\documentclass[11pt]{article}

\usepackage[utf8]{inputenc}
\usepackage[T1]{fontenc}
\usepackage{amsmath}
\usepackage{amssymb}
\usepackage{amsfonts}
\usepackage{mathtools}
\usepackage{amsthm}
\usepackage{bm}
\usepackage{graphicx}
\usepackage{url}
\usepackage{booktabs}
\usepackage{nicefrac}
\usepackage{microtype}
\usepackage[dvipsnames]{xcolor}
\definecolor{hunyuanblue}{HTML}{1E4A8F}
\usepackage{setspace}
\usepackage[font=small,labelfont=bf,tableposition=top]{caption}
\usepackage[font=scriptsize,skip=1pt,subrefformat=parens]{subcaption}
\usepackage{colortbl}
\usepackage{ragged2e}
\usepackage{makecell}
\usepackage{tabularx}
\usepackage{textcomp}
\usepackage{algorithm}
\usepackage{algpseudocode}
\usepackage{tabu}
\usepackage[shortlabels]{enumitem}
\usepackage{multirow}
\usepackage{wrapfig}
\usepackage{pifont}
\usepackage{bbding}
\usepackage{arydshln}
\usepackage{fontawesome}
\usepackage{footnote}
\usepackage{tablefootnote}
\usepackage{soul}
\usepackage[most]{tcolorbox}
\usepackage{placeins}
\usepackage{xspace}
\usepackage{fancyhdr}

\usepackage{hunyuan}
\usepackage[capitalize,noabbrev]{cleveref}

\makeatletter
\long\def\@makecaption#1#2{%
  \vskip 10pt
  \setbox\@tempboxa\hbox{#1: #2}%
  \ifdim \wd\@tempboxa >\hsize
    \noindent #1: #2\par
  \else
    \hbox to\hsize{\hfil\box\@tempboxa\hfil}%
  \fi}
\makeatother

\hypersetup{
  colorlinks=true,
  linkcolor=hunyuanblue,
  citecolor=hunyuanblue,
  urlcolor=hunyuanblue,
}

\makeatletter
\def\section{\@startsiction{section}{1}{\z@}{-0.24in}{0.10in}
             {\large\bf\raggedright\color{hunyuanblue}}}
\def\subsection{\@startsection{subsection}{2}{\z@}{-0.20in}{0.08in}
                {\normalsize\bf\raggedright\color{hunyuanblue}}}
\makeatother

\definecolor{abstractbg}{HTML}{F0F7FC}
\definecolor{colRed}{HTML}{C00000}
\definecolor{colGreen}{HTML}{95BAA6}
\definecolor{colBlue}{HTML}{4B8DBC}
\definecolor{colYellow}{HTML}{E2CD89}
\definecolor{colSky}{HTML}{A7CAEA}
\definecolor{colOlive}{HTML}{A2CD5A}
\definecolor{colPurple}{HTML}{800080}
\definecolor{colOrange}{HTML}{FFA500}
\definecolor{colAzure}{HTML}{F0FFFF}
\definecolor{colGray}{gray}{0.5}
\definecolor{mygreen}{HTML}{95BAA6}
\definecolor{mydarkgreen}{rgb}{0.02,0.6,0.02}

\fancypagestyle{plain}{%
  \fancyhf{}%
  \fancyfoot[C]{\thepage}%
}
\fancypagestyle{firststyle}{%
  \fancyhf{}%
  \fancyhead[L]{\raisebox{-18.5mm}[0pt][0pt]{\includegraphics[height=12mm]{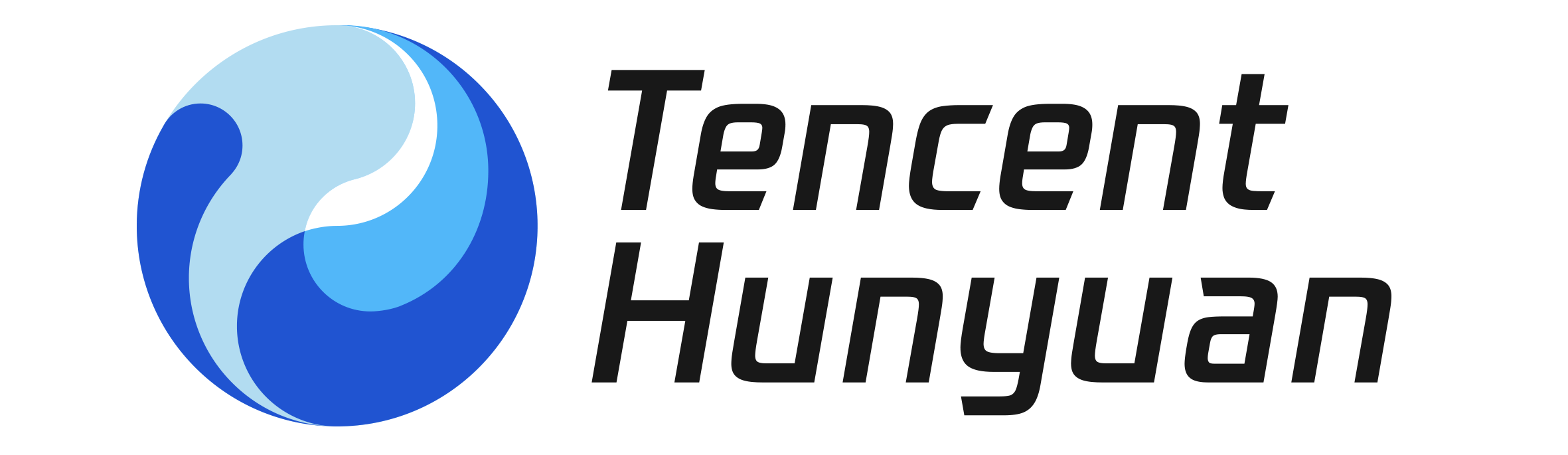}}}%
  \fancyfoot[C]{\thepage}%
}

\newcolumntype{Y}{>{\raggedright\arraybackslash}X}

\newcommand{\txtGreen}[1]{\textcolor{colGreen}{#1}}
\newcommand{\txtBlue}[1]{\textcolor{colBlue}{#1}}
\newcommand{\txtYellow}[1]{\textcolor{colYellow}{#1}}

\newcommand{\bx}[0]{{\boldsymbol{x}}}
\newcommand{\bmu}[0]{{\boldsymbol{\mu}}}
\newcommand{\bv}[0]{{\boldsymbol{v}}}

\newcommand{\bepsilon}[0]{{\boldsymbol{\epsilon}}}

\tcbset{
  greenhighlight/.style={
    colback=mygreen!10,
    colframe=mygreen,
    boxrule=0.8pt,
    arc=1.5mm,
    left=1mm, right=1mm,
    top=0.5mm, bottom=0.5mm,
    fonttitle=\bfseries,
    enhanced,
    rounded corners
  }
}

\theoremstyle{definition}

\usepackage{amsmath,amsfonts,bm}

\usepackage{pifont}

\def\eqref#1{(\ref{#1})}

\def\1{\bm{1}}

\def\rmI{{\mathbf{I}}}

\DeclareMathAlphabet{\mathsfit}{\encodingdefault}{\sfdefault}{m}{sl}
\SetMathAlphabet{\mathsfit}{bold}{\encodingdefault}{\sfdefault}{bx}{n}

\newcommand{\E}{\mathbb{E}}

\DeclareMathOperator*{\argmin}{arg\,min}

\usepackage{graphicx}
\usepackage{url}            %
\usepackage{booktabs}       %
\usepackage{nicefrac}       %
\usepackage{microtype}      %
\usepackage{wrapfig}

\usepackage{enumitem}

\usepackage{multirow}
\usepackage{tablefootnote}

\usepackage{color}
\usepackage{colortbl}
\usepackage{xcolor}

\definecolor{blgrey}{rgb}{0.6,0.6,0.6}
\definecolor{bblue}{rgb}{0.855,0.933,0.98}
\definecolor{dblue}{HTML}{5297D6}
\definecolor{gainred}{rgb}{0.1,0.5,0.3}
\definecolor{citecolor}{HTML}{0071BC}
\definecolor{linkcolor}{HTML}{ED1C24}

\usepackage{listings}
\usepackage{color}

\definecolor{dkcyan}{cmyk}{1,0,0,.25}
\definecolor{dkgreen}{rgb}{0,0.6,0}
\definecolor{gray}{rgb}{0.5,0.5,0.5}
\definecolor{mauve}{rgb}{0.58,0,0.82}

\newcommand{\DKL}{\mathbb{D}_{\mathrm{KL}}}

\newcommand{\papertitle}{Reinforcing Few-step Generators via\\Reward-Tilted Distribution Matching}
\newcommand{\paperauthors}{%
  \textbf{Yushi Huang}$^{1,2,*}$ \quad
  \textbf{Xiangxin Zhou}$^{1,*\,\dagger}$ \quad
  \textbf{Ruoyu Wang}$^{1,3,*}$\\[4pt]
  \textbf{Chi Zhang}$^{3}$ \quad
  \textbf{Jun Zhang}$^{2}$ \quad
  \textbf{Tianyu Pang}$^{1,\ddagger}$%
}
\newcommand{\paperaffiliations}{%
  $^1$Tencent Hunyuan \quad $^2$Hong Kong University of Science and Technology \\[2pt]
  $^3$Westlake University \\[6pt]
  {\small $^*$Equal contribution \quad $^\dagger$Project Lead \quad $^\ddagger$Corresponding author}%
}

\begin{document}

\thispagestyle{firststyle}
\vspace*{0.25cm}
{\color{hunyuanblue}\hrule height 0.6pt}
\vskip 6mm
\begin{center}
{\LARGE\bfseries \papertitle\par}
\end{center}
\vskip 3mm
{\color{hunyuanblue}\hrule height 0.6pt}
\vskip 6mm
\begin{center}
\paperauthors\\[4pt]
{\small \paperaffiliations}
\end{center}
\vskip 6mm

\begin{tcolorbox}[
  colframe=abstractbg,
  colback=abstractbg,
  boxrule=0pt,
  arc=2mm,
  enhanced,
  top=12pt,
  bottom=12pt,
  left=15pt,
  right=15pt,
  width=\textwidth,
]
\textbf{Abstract.}\quad

Recent advances in few-step diffusion distillation have enabled efficient image generation, yet aligning these models with human preferences remains challenging. We propose \emph{Reward-Tilted Distribution Matching Distillation} (RTDMD), a two-stage framework that unifies distribution matching distillation
with reward-guided reinforcement learning for few-step flow generators. We show that minimizing the KL divergence to a reward-tilted teacher distribution naturally decomposes into a distribution matching term and a reward maximization term.
In the first stage, we introduce \emph{Ambient-Consistent Distribution Matching Distillation} (AC-DMD), which performs subinterval-wise distribution matching and augments the fake score objective with a consistency regularizer to help the fake score model track the shifting generator distribution under limited updates. In the second stage, we jointly optimize both terms: for the reward maximization term, we derive a hybrid policy gradient that combines a GRPO-style estimator for the stochastic intermediate transitions with direct reward backpropagation through the deterministic final step, and further introduce \emph{step-subset GRPO} (SubGRPO) to reduce variance. Experiments on SD3, SD3.5, and FLUX.2 demonstrate that RTDMD establishes new state-of-the-art results across preference, aesthetic, and compositional metrics with only 4 inference steps, outperforming previous few-step text-to-image generation methods. Code and models are available at \url{https://github.com/Harahan/RTDMD}.

\vskip 8pt
\textbf{Date:} \today
\end{tcolorbox}

\section{Introduction}

Diffusion~\citep{ddpm,score-base-diff} and flow-based generative models~\citep{flow_matching,liu2022flow} have achieved remarkable progress in text-to-image generation.
Modern diffusion and rectified-flow systems~\citep{sd3,flux,sdxl} can synthesize realistic and semantically aligned images, but their iterative sampling procedures typically require tens of denoising or flow-integration steps~\citep{ddim,ddpm}.
This high sampling cost limits their deployment in latency-sensitive applications such as interactive content creation, on-device generation, and real-time visual systems.

\begin{figure}[!ht]
   \centering
     \includegraphics[width=\textwidth]{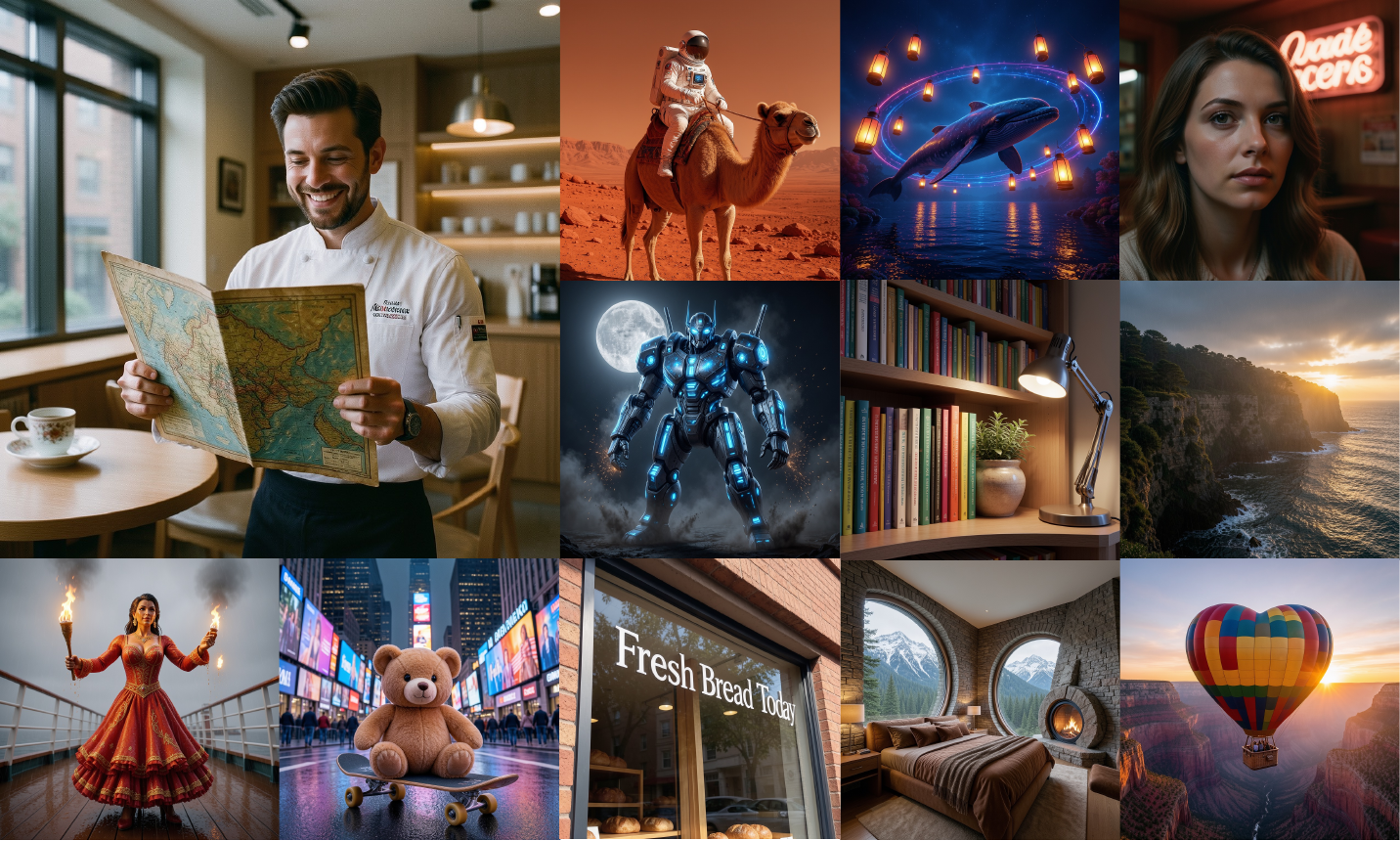}
     \caption{Visual generations produced by our RTDMD method under 4 NFE on FLUX.2 4B~\citep{flux-2} without applying classifier-free guidance (CFG)~\citep{ho2022classifierfreediffusionguidance}. More visual results can be found in App.~\ref{app:more_qualitative}.}
    \label{fig:teaser}
\end{figure}

To improve efficiency, recent works distill pretrained multi-step models into few-step generators~\citep{salimans2022progressive,lcm,sauer2024adversarial,sauer2024fast,tdm,flash-dmd,fan2026phaseddmdfewstepdistribution}.
Among them, Distribution Matching Distillation (DMD)~\citep{dmd,dmd2,decoupled-dmd,senseflow} trains a student to match the teacher's output distribution via a learned fake score model.
Orthogonally, reinforcement learning (RL) aligns generative models with human preferences~\citep{ddpo,fan2023dpok,flowgrpo,dancegrpo,xu2023imagereward,hpsv2,clark2023directly,wallace2024diffusion}.
Recent efforts combine distribution matching with reward optimization~\citep{dmdr,luo2026tdmr1reinforcingfewstepdiffusion,dong2026guidingdistributionmatchingdistillation,fan2026rtextdmreconceptualizingdistributionmatching}, aiming to retain the teacher's generative prior while steering the student toward higher-reward outputs.

However, reward-guided few-step generation remains challenging for two reasons.
First, in few-step generation, the intermediate latents at non-terminal timesteps are inherently noisy. The fake score model in DMD must therefore be trained on these noisy intermediates rather than clean samples. Moreover, the generator distribution shifts at every training iteration, requiring the fake score to continuously track a moving target under a limited compute budget, which makes the cold-start distillation signal unreliable.
Second, reward optimization must respect the hybrid nature of the sampling dynamics: intermediate steps are stochastic due to injected noise, while the final step is deterministic (terminal noise level is zero). Optimizing only the stochastic steps~\citep{flowgrpo,mixgrpo} or only the deterministic final mapping~\citep{xu2023imagereward,dmdr} are both suboptimal; a tailored estimator that accounts for the full trajectory is needed.

In this work, we propose \emph{Reward-Tilted Distribution Matching Distillation} (RTDMD), a two-stage framework for training high-quality few-step flow generators.
We show that minimizing the KL divergence to a \textit{reward-tilted teacher distribution} (defined in Eq.~\eqref{eq:reward_tilted_teacher_distribution}) naturally decomposes into a distribution matching term and a reward maximization term, providing a principled unification of distillation and RL.
In the first stage, we introduce \emph{Ambient-Consistent DMD} (AC-DMD) as a stable cold start. AC-DMD performs distribution matching on each time subinterval independently, and augments the fake score objective with a consistency regularizer~\citep{daras2023consistent,daras2024consistent} that couples predictions across timesteps. This helps the fake score model track the shifting generator distribution more effectively under limited updates.
In the second stage, we jointly optimize both terms via a hybrid policy gradient that combines GRPO-style updates for stochastic intermediate transitions with direct reward backpropagation through the deterministic final step, and further introduce step-subset GRPO (SubGRPO) with shared noise to reduce variance.

Comprehensive experiments on SD3-M~\citep{sd3}, SD3.5-M~\citep{sd35}, and FLUX.2 4B~\citep{flux-2} demonstrate that RTDMD achieves state-of-the-art few-step generation quality under 4-step sampling. Notably, our distilled FLUX.2 4B surpasses the full FLUX.2 9B (50-step) across most benchmarks.

\section{Preliminaries}

\subsection{Diffusion and Flow Models}
Diffusion and flow-based generative models~\citep{ddim, ddpm} define a continuous probability path $\{p_t\}_{t\in[0,1]}$ that connects the data distribution $p_0$ to a simple prior $p_1$ (typically a standard Gaussian). A common Gaussian interpolation is $q_t(\bx_t\mid\bx_0)=\mathcal{N}(\alpha_t\bx_0,\sigma_t^2\mathbf{I})$,
where $\bx_0\sim p_0$ and $\bepsilon\sim\mathcal{N}(\mathbf{0},\mathbf{I})$, so that $\bx_t=\alpha_t\bx_0+\sigma_t\bepsilon$ and $p_t(\bx_t)=\int q_t(\bx_t\mid\bx_0)p_0(\bx_0)\,d\bx_0$. Here $\alpha_t$ and $\sigma_t$ specify the noise schedule.

Sampling is described by the probability-flow Ordinary Differential Equation (PF-ODE)~\citep{flow_matching}
$\frac{\mathrm{d}\bx_t}{\mathrm{d}t}=\bv(\bx_t,t)$,
where $\bv(\bx,t)$ is the \emph{marginal velocity field} transporting the density $p_t$. Its relation~\citep{sit} to the \emph{score function}
$s(\bx,t)=\nabla_{\bx}\log p_t(\bx)$ is
\begin{equation}
\bv(\bx,t)
=
\frac{\dot{\alpha}_t}{\alpha_t}\bx
+
\left(\frac{\dot{\alpha}_t}{\alpha_t}\sigma_t^2-\dot{\sigma}_t\sigma_t\right)s(\bx,t).
\end{equation}
Thus, the score function indicates how the density changes locally, while the marginal velocity determines how samples move along the probability path.

In this work, we adopt flow matching with the rectified schedule~\citep{flow_matching,sd3}, namely $\alpha_t=1-t$ and $\sigma_t=t$, which gives the linear path $\bx_t=(1-t)\bx_0+t\bepsilon$. For a fixed pair $(\bx_0,\bepsilon)$, the conditional velocity is simply $\bv_t(\bx_t\mid\bx_0,\bepsilon)=\bepsilon-\bx_0$, and the marginal velocity satisfies
\begin{equation}
\bv(\bx,t)=\mathbb{E}[\bepsilon-\bx_0\mid \bx_t=\bx]
=-\frac{1}{1-t}\bx-\frac{t}{1-t}s(\bx,t).
\label{eq:v_s}
\end{equation}

Flow matching trains a neural velocity field $\bv_\theta$ by regressing it to the conditional target velocity~\citep{flow_matching}:
\begin{equation}
\mathcal{L}_{\mathrm{CFM}}(\theta)
=
\mathbb{E}_{t,\bx_0,\bepsilon}
\bigl[
w(t)\|\bv_\theta(\bx_t,t)-(\bepsilon-\bx_0)\|_2^2
\bigr],
\quad \bx_t=(1-t)\bx_0+t\bepsilon,
\end{equation}
where $w(t)$ is an optional weighting function. This objective is commonly referred to as the conditional flow matching (CFM) loss.

\subsection{Distribution Matching Distillation}

Sampling from a pretrained flow model typically requires many function evaluations, motivating distillation into a few-step generator.
Let $\bv_\psi$ denote the pretrained \emph{teacher} velocity field, and let $p_\psi$ be its induced distribution.

Distribution Matching Distillation (DMD)~\citep{dmd,dmd2} trains a few-step \emph{student} generator $G_\theta$ with $\bx_0 = G_\theta(\bepsilon)$, where $\bepsilon \sim \mathcal{N}(\mathbf{0},\mathbf{I})$, so that its induced distribution $p_\theta$ matches $p_\psi$.
A natural objective is the reverse Kullback–Leibler (KL) $\DKL(p_\theta \,\|\, p_\psi)$.
Because this divergence can be difficult to optimize directly in data space when the two distributions have limited overlap, DMD instead compares their noised marginals in an ambient space.
Specifically, for $t \sim \mathrm{Unif}[0,1]$ and an independent $\bepsilon' \sim \mathcal{N}(\mathbf{0},\mathbf{I})$, it defines
$\bx_t = (1-t)\bx_0 + t\bepsilon'$,
which induces marginals $p_{\theta,t}$ and $p_{r,t}$ for the student and teacher, respectively.
The resulting time-averaged reverse KL yields a generator gradient proportional to the difference between the student and teacher scores at $\bx_t$.

Using Eq.~\eqref{eq:v_s} to convert between velocities and scores, DMD writes the generator update as
\begin{equation}
\label{eq:dmd_s}
\nabla_\theta\mathcal{L}_{\text{DMD}}
=
\mathbb{E}_{\bepsilon,t,\bepsilon'}\big[
w_t \alpha_t \bigl(s_\phi(\bx_t,t)-s_\psi(\bx_t,t)\bigr)\nabla_\theta G_\theta(\bepsilon)
\big],
\end{equation}
where $w_t>0$ is a time-dependent weight and $\alpha_t = 1-t$ comes from $\partial \bx_t/\partial \bx_0$.

Since the student score is not available in closed form, DMD introduces an auxiliary \emph{fake} velocity field $\bv_\phi$ to track the current student distribution $p_\theta$.
Via Eq.~\eqref{eq:v_s}, $\bv_\phi$ defines the fake score $s_\phi$, which serves as a surrogate for the student score in Eq.~\eqref{eq:dmd_s}.
The fake velocity is trained with the conditional flow-matching objective
\begin{equation}
\label{eq:cfm}
\mathcal{L}_f(\phi)
=
\mathbb{E}_{\bepsilon,t,\bepsilon'}\big[
\|\bv_\phi(\bx_t,t)-(\bepsilon'-\bx_0)\|^2
\big].
\end{equation}
At optimum, $\bv_\phi$ recovers the marginal velocity field of the current student distribution, providing the score estimate required by the DMD update.
DMD therefore alternates between updating the student generator $G_\theta$ and training the fake model $\bv_\phi$ to track it.

\section{Method}

We present \emph{Reward-Tilted Distribution Matching Distillation} (RTDMD), a principled framework for training high-quality few-step generators (an overall algorithm can be found in App.~\ref{app:algo}).
Let $p_\psi$ denote the distribution induced by the pretrained teacher model. 
While DMD~\citep{dmd, dmd2} aims to replicate $p_\psi$, the teacher distribution itself is not necessarily aligned with human preferences, which means it can assign equal probability to both high-reward and low-reward samples. A natural remedy is to up-weight high-reward regions of $p_\psi$ while down-weighting low-reward ones. Therefore, we define the \textit{reward-tilted teacher distribution} as
\begin{equation}
    \tilde{p}_\psi(\bx)=\frac{p_\psi(\bx)\exp(\beta r(\bx))}{Z},
    \label{eq:reward_tilted_teacher_distribution}
\end{equation}
where $r(\bx)$ is a scalar reward function, $\beta\ge 0$ controls the reward strength, and $Z$ is the normalizing constant.
We optimize the few-step generator $G_\theta$ by minimizing
$\DKL(p_\theta\|\tilde{p}_\psi)$.
Since $\log \tilde{p}_\psi(\bx)=\log p_\psi(\bx)+\beta r(\bx)-\log Z$, and $Z$ is independent of $\theta$, we have $ \DKL(p_\theta\|\tilde{p}_\psi)=\DKL(p_\theta\| {p}_\psi) - \beta \E_{\hat{\bx}_0\sim p_\theta}[r(\hat{\bx}_0)] + \log Z$ and
\begin{equation}
\nabla_\theta \DKL(p_\theta\|\tilde{p}_\psi)
=
\underbrace{\nabla_\theta \DKL(p_\theta\| {p}_\psi)}_{\text{distribution matching}}
-
\beta \underbrace{\nabla_\theta \E_{\hat{\bx}_0\sim p_\theta}[r(\hat{\bx}_0)]}_{\text{reward maximization}}.
\label{eq:rtdmd-decomposition}
\end{equation}
This decomposition shows that minimizing the KL to the reward-tilted distribution naturally separates into a \emph{distribution matching} term and a \emph{reward maximization} term. This motivates our two-stage framework: we first perform distribution matching as a cold start via \textit{Ambient-Consistent DMD} (AC-DMD, Sec.~\ref{sec:ac-dm}), and then jointly optimize both terms using a \textit{hybrid policy gradient} with \textit{step-subset GRPO} for the reward term (Sec.~\ref{sec:reinforce_gen}).

\begin{figure}[!ht]
   \centering
     \includegraphics[width=0.85\textwidth]{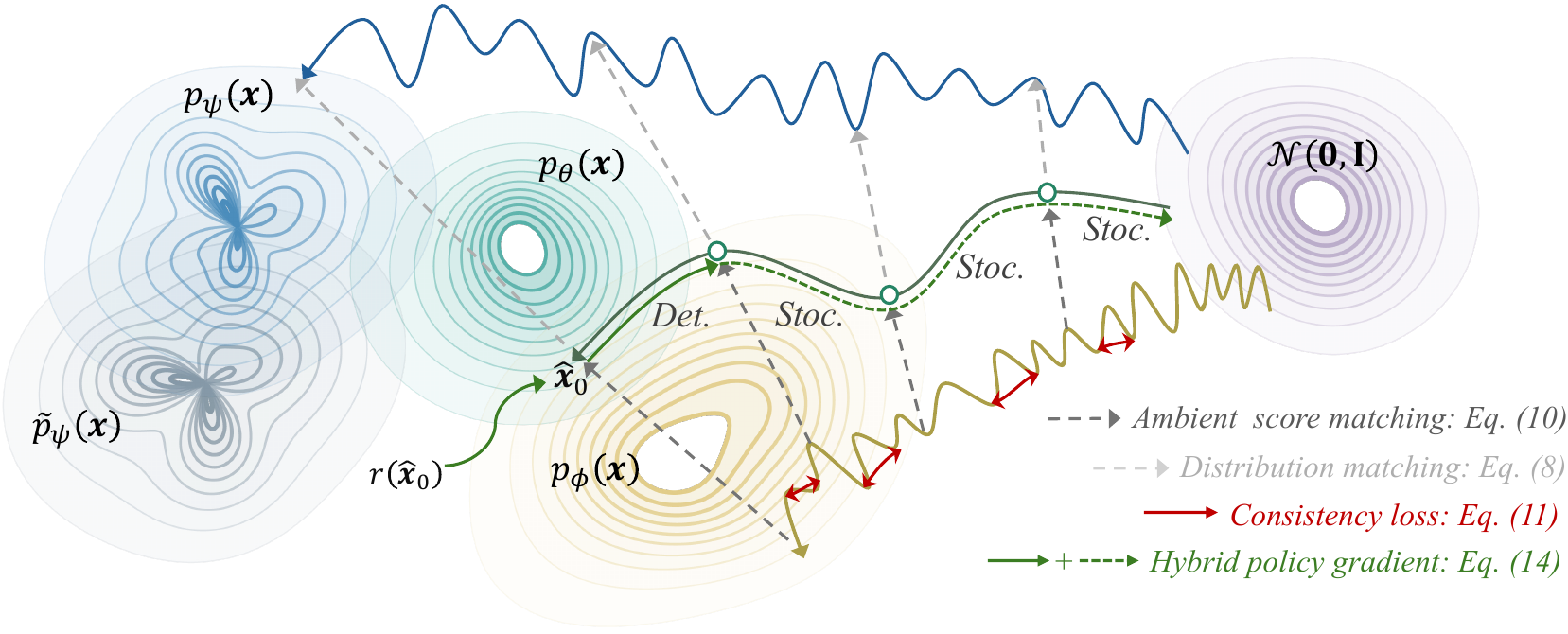}
     \caption{Overview of the proposed RTDMD. ``\textit{Det.}'' means the final deterministic step and ``\textit{Stoc.}'' denotes the stochastic steps (see Sec.~\ref{sec:reinforce_gen}). \txtBlue{Blue}, \txtGreen{green}, and \txtYellow{yellow} trajectories represent the denoising trajectory of the pretrained-teacher, few-step generator, and fake score model, respectively.}
    \label{fig:overview}
\end{figure}

\subsection{Ambient-Consistent Distribution Matching Distillation}\label{sec:ac-dm}
Existing DMD methods adopt either the deterministic Euler ODE sampler~\citep{fan2026phaseddmdfewstepdistribution} or the consistency model (CM) sampler~\citep{dmd, dmd2} for few-step generation. To unify these choices under a single framework and facilitate the subsequent policy-gradient derivation (Sec.~\ref{sec:reinforce_gen}), we first employ coefficient-preserving sampling (CPS)~\citep{wang2025coefficients}, which encompasses both as special cases by a predefined hyperparameter $\eta$ (see App.~\ref{app:cps_formula} for the full formula).
Under CPS, each generation step consists of a denoising prediction followed by noise injection, and it also ensures that the noise level of the latent variable remains consistent with the predefined scheduler at every timestep. 

To be more specific, we use a $K$-step generator with $K=4$ and a decreasing timestep schedule $0=t_K<\cdots<t_1<t_0=1$.
Starting from $\hat{\bx}_{t_0}\sim\mathcal N(\mathbf 0,\mathbf I)$, step $k$ takes the current latent $\hat{\bx}_{t_{k-1}}$ and outputs an $x$-prediction, $\hat{\bx}_{\mathrm{pred}}^{(k)} = G_\theta(\hat{\bx}_{t_{k-1}}, t_{k-1}),\,k=1,\dots,K,$
which is the sampler output under the $x$-parameterization, rather than a clean sample itself. The next latent $\hat{\bx}_{t_k}$ is a linear combination of the $x$-prediction $\hat{\bx}_{\mathrm{pred}}^{(k)}$, the current latent $\hat{\bx}_{t_{k-1}}$, and a freshly sampled Gaussian noise $\bepsilon_k$. Here, $\eta\in[0,1]$ controls the sampling stochasticity: $\eta=0$ recovers the deterministic Euler sampler, while $\eta>0$ injects noise at each step.

\noindent\textbf{Ambient distribution matching distillation.} Since CPS ($\eta>0$) injects noise at intermediate steps, the generator output $\hat{\bx}_{t_k}$ at $t_k > 0$ could no longer be a clean sample but a noisy latent at noise level $t_k$. The standard   
DMD~\citep{dmd, dmd2}, which assumes clean samples and performs score matching over the full interval $[0,1]$, is therefore no longer directly applicable. We re-derive the distribution
matching objective on the subinterval $[t_k, 1]$ conditioned on the noisy intermediate, and term this \emph{Ambient Distribution Matching Distillation} (A-DMD). 

Concretely, let $p_\theta^{(k)}$ denote the distribution of $\hat{\bx}_{t_k}$ after $k$ steps. To train step $k$, we match the teacher distribution on the subinterval $[t_k,1]$. Under the rectified schedule~\citep{flow_matching,sd3}, we re-noise $\hat{\bx}_{t_k}$ to any level $t \in [t_k, 1]$ via $\bx_t^{(k)} = \alpha_k(t)\hat{\bx}_{t_k} + \sigma_k(t)\bepsilon$, where $\alpha_k(t) = \frac{1-t}{1-t_k}$, $\sigma_k(t) = \frac{t-t_k}{1-t_k}$, and $\bepsilon \sim \mathcal{N}(\mathbf{0}, \mathbf{I})$. Let $p_{\theta,t}^{(k)}$ denote the resulting student marginal at noise level $t$. We minimize the reverse KL:
\begin{equation}
\mathcal L_{\mathrm{gen}}^{(k)}(\theta)
=
\mathbb E_{t\sim\mathcal U[t_k,1]}
\bigl[\lambda_k(t)\,\mathrm{KL}(p_{\theta,t}^{(k)}\,\|\,p_{\psi,t})\bigr],
\label{eq:gen_kl}
\end{equation}
where $p_{\psi,t}$ is the teacher marginal and $\lambda_k(t)$ is a timestep-dependent weight.

Since the student score $\nabla_{\bx}\log p_{\theta,t}^{(k)}(\bx)$ is intractable, following DMD \citep{dmd,dmd2}, we introduce a fake score model $s_\phi(\bx,t)$ to approximate it, yielding the practical generator gradient
\begin{equation}
\nabla_\theta \mathcal L_{\mathrm{gen}}^{(k)}
\approx
-
\mathbb E_{t,\bepsilon,\hat{\bx}_{t_k}}
\left[
\lambda_k(t)\alpha_k(t)\,
\bigl(
s_\psi(\bx_t^{(k)},t)-s_\phi(\bx_t^{(k)},t)
\bigr)^\top
\frac{\partial \hat{\bx}_{t_k}}{\partial \theta}
\right].
\label{eq:subinterval-generator-gradient}
\end{equation}
This form (see App.~\ref{app:k_step_objective} for the detailed derivation) makes the training signal entirely local to the subinterval $[t_k,1]$: the teacher score $s_\psi(\bx_t^{(k)},t)$ provides the target direction, while the fake score $s_\phi(\bx_t^{(k)},t)$ compensates for the intractable student marginal score.

To train $s_\phi$, we fit it on the same interval $[t_k,1]$ using denoising score matching (DSM):
\begin{equation}
\mathcal L_{\mathrm{fake}}^{(k)}(\phi)
=
\mathbb E_{t\sim\mathcal U[t_k,1],\,\hat{\bx}_{t_k},\,\bepsilon}
\left[
\omega_k(t)\,
\left\|
s_\phi(\bx_t^{(k)}, t)
-
\nabla_{\bx}\log q_t^{(k)}(\bx \mid \hat{\bx}_{t_k})
\right\|_2^2
\right],
\label{eq:ambient_score_matching}
\end{equation}
where
$\nabla_{\bx}\log q_t^{(k)}(\bx \mid \hat{\bx}_{t_k})
=
\left(\alpha_k(t)\hat{\bx}_{t_k}-\bx\right)/\sigma_k(t)^2$ is the conditional score of the Gaussian corruption kernel, and $\omega_k(t)\!>\!0$ is a timestep-dependent weight (a design choice independent of $\lambda_k(t)$ in Eq.~\eqref{eq:gen_kl}).

\noindent\textbf{Stabilizing fake score training via consistency regularization.} However, when $t_k>0$, the fake score model $s_\phi$ is trained on corrupted intermediate latents $\hat{\bx}_{t_k}$ rather than clean samples. Although the DSM objective in Eq.~\eqref{eq:ambient_score_matching} is theoretically unbiased (its optimal solution is the true student marginal score $\nabla_\bx \log p_{\theta,t}^{(k)}$ as proved in App.~\ref{app:dsm_optimality}), a practical challenge arises: the generator $G_\theta$ is updated concurrently, so the target distribution $p_\theta^{(k)}$ shifts at every training iteration. With only a limited number of fake score updates per generator step, $s_\phi$ must track this moving distribution under a tight sample and compute budget, making accurate estimation difficult.

To stabilize fake score training, we introduce a consistency regularizer~\citep{daras2023consistent,daras2024consistent}. The key insight is that the optimal fake score model
satisfies a self-consistency property (see App.~\ref{app:consistency_proof} for a detailed proof): for any $t'' < t'$, its $x$-prediction at $t'$ must equal the expected $x$-prediction after one reverse-diffusion step to $t''$,
\emph{i.e.}, $\hat{\bx}_\phi(\bx_{t'}, t') = \mathbb{E}_{\tilde{\bx}_{t''} \sim p_\phi(\bx_{t''}\mid \bx_{t'})}[\hat{\bx}_\phi(\tilde{\bx}_{t''}, t'')]$, where $\hat{\bx}_\phi(\bx, t)$ denotes the fake score model's $x$-prediction. This couples the fake score predictions across different timesteps, reducing the effective degrees of freedom and lowering the overall estimation variance.                              

Concretely, writing the fake score model in its $x$-prediction form, we penalize violations of this property:
\begin{equation}                                                                                                                                                                       
\!\!\!\!\!\mathcal L_{\mathrm{cons}}^{(k)}(\phi)
=                                                                                                                                                                                      
\mathbb E_{t',t'',\,\hat{\bx}_{t_k}\!\sim p_\theta^{(k)},\,\bx_{t'}\!\sim q_{t'}^{(k)}(\cdot|\hat{\bx}_{t_k})}
\!\left[        
\left\|                                                   
\hat{\bx}_{\phi}(\bx_{t'},t')
-
\mathbb E_{\tilde{\bx}_{t''}\!\sim p_\phi(\bx_{t''}\mid \bx_{t'})}
\!\bigl[
\hat{\bx}_{\phi}(\tilde{\bx}_{t''},t'')
\bigr]
\right\|_2^2
\right],
\label{eq:consistency_loss}
\end{equation}
where $p_\phi(\bx_{t''}\mid \bx_{t'})$ is the fake score model's reverse transition kernel from $t'$ to $t''$. In practice, we use an approximated estimator following \citet{daras2024consistent} (see App.~\ref{app:consistency_practice}). We choose $t'$ and $t''$ to be close so that the consistency term remains local and can be estimated efficiently with a single transition step. The final fake-score objective is
\begin{equation}
\mathcal L_{\mathrm{fake\mbox{-}total}}^{(k)}(\phi)
=
\mathcal L_{\mathrm{fake}}^{(k)}(\phi)
+
\gamma\,\mathcal L_{\mathrm{cons}}^{(k)}(\phi).
\label{eq:fake_score_total}
\end{equation}

Intuitively, Eq.~\eqref{eq:ambient_score_matching} provides pointwise score supervision at each noise level, while the consistency term couples nearby timesteps by requiring them to predict the same underlying clean sample, thereby reducing the variance of ambient fake-score training.

Overall, we refer to our method as \emph{Ambient-Consistent Distribution Matching Distillation} (AC-DMD), reflecting that the fake score model is trained on noisy intermediate latents (the ``ambient'' setting) and regularized by a consistency loss to improve estimation quality.

\subsection{Reinforcing the Few-step Generator}\label{sec:reinforce_gen}
After the cold start with AC-DMD, we proceed to the second stage: jointly optimizing both terms in Eq.~\eqref{eq:rtdmd-decomposition}. The distribution matching term is handled by AC-DMD as before; we now focus on deriving efficient gradient estimators for the reward maximization term.

\noindent\textbf{Few-step generator as a policy.}
The few-step generator induces a $K$-step policy over the latent trajectory
$\hat{\bx}_{t_0}\!\to\!\hat{\bx}_{t_1}\!\to\!\cdots\!\to\!\hat{\bx}_{t_K}$.
At each step, the CPS update combines the generator's $x$-prediction with the current latent and injected Gaussian noise~\footnote{We set $\eta > 0$ throughout this work, as it naturally introduces stochasticity into the sampling trajectory, which is essential for exploration in reinforcement learning.}.
As a result, for the first $K-1$ steps, the transition defines a Gaussian policy
\begin{equation}
\pi^{(k)}_\theta(\hat{\bx}_{t_k}\mid \hat{\bx}_{t_{k-1}})
=
\mathcal N\!\bigl(\hat{\bx}_{t_k};\,\bmu^{(k)}_\theta(\hat{\bx}_{t_{k-1}}),\,\sigma_k^2 \rmI\bigr),
\qquad k=1,\dots,K-1,
\end{equation}
where $\bmu^{(k)}_\theta$ is determined by the CPS update (App.~\ref{app:cps_formula}) and $\sigma_k = t_k \sin(\eta\pi/2)$.
The final step is deterministic: since $t_K=0$, the noise term vanishes and $\hat{\bx}_0 = G_\theta(\hat{\bx}_{t_{K-1}}, t_{K-1})$.
Therefore, the few-step generative process is a hybrid policy consisting of $K-1$ stochastic Gaussian steps followed by one deterministic step.

\noindent\textbf{Hybrid policy gradient.} As a result, the reward gradient (\emph{i.e.}, $\nabla_\theta \E_{\hat{\bx}_0\sim p_\theta}[r(\hat{\bx}_0)]$ in Eq.~\eqref{eq:rtdmd-decomposition}) naturally decomposes into a contribution from the stochastic intermediate transitions and a contribution from the deterministic final mapping.

Specifically, let
\(
\tau=(\hat{\bx}_{t_0},\hat{\bx}_{t_1},\dots,\hat{\bx}_{t_{K-1}},\hat{\bx}_0)
\)
denote a generated trajectory. Then
\begin{equation}
\begin{aligned}
\nabla_\theta \E_{\hat{\bx}_0\sim p_\theta}[r(\hat{\bx}_0)]
&=
\sum_{k=1}^{K-1}
\E_{\tau\sim p_\theta}\!\left[
r(\hat{\bx}_0)\,
\nabla_\theta \log \pi_\theta^{(k)}(\hat{\bx}_{t_k}\mid \hat{\bx}_{t_{k-1}})
\right] \\
&\quad+
\E_{\tau\sim p_\theta}\!\left[
\bigl(\nabla_{\hat{\bx}_0} r(\hat{\bx}_0)\bigr)^\top
\partial_\theta G_\theta(\hat{\bx}_{t_{K-1}}, t_{K-1})
\right].
\end{aligned}
\label{eq:hybrid-policy-gradient}
\end{equation}
The first term is a REINFORCE-style estimator and accounts for how the parameters affect the distribution of the stochastic intermediate states, while the second term differentiates the deterministic final denoising step (a formal derivation is provided in App.~\ref{app:hybrid_pg}; see Prop.~\ref{prop:hybrid_pg}). Since the reward is typically differentiable, we estimate the second term by directly backpropagating through $G_\theta$:
\begin{equation}               
\nabla_\theta \mathcal{L}_{\mathrm{det}}                  
=
\frac{1}{N}\sum_{i=1}^N
\bigl(\nabla_{\hat{\bx}_0} r(\hat{\bx}_0^{(i)})\bigr)^\top
\partial_\theta G_\theta(\hat{\bx}_{t_{K-1}}^{(i)}, t_{K-1}).
\label{eq:det_grad}
\end{equation}
For the first term, directly using the REINFORCE-style term leads to high variance. Following GRPO~\citep{deepseekmath}, we reduce variance of it by sampling a group of $N$ trajectories $\{\tau_i\}_{i=1}^N$ per prompt and replacing the raw reward $r$ with a group-normalized advantage $A_i = (r_i - \bar{r}) / \operatorname{std}(\{r_j\}_{j=1}^N$, where $r_i = r(\hat{\bx}_0^{(i)})$ is the reward of the $i$-th trajectory and $\bar{r} = \frac{1}{N}\sum_{j=1}^N r_j$.

\noindent\textbf{Step-subset GRPO with shared noise.} However, naive GRPO (see App.~\ref{app:grpo_background} for more details) uses independent noise at every step, so reward differences across trajectories conflate contributions from all steps. Inspired by MixGRPO~\citep{mixgrpo}, we propose \emph{step-subset GRPO with shared noise} (SubGRPO) to further reduce variance by isolating the effect of selected steps. For each prompt, we uniformly sample a subset of stochastic steps
\(
\mathcal{S} \subset \{1,\dots,K-1\}
\)
with \(|\mathcal{S}| = M < K-1\).
The full \(K\)-step trajectory is still rolled out, but only the steps in \(\mathcal{S}\) use independent noise across trajectories; the remaining steps share noise within the group:
\begin{equation}
\tilde{\bepsilon}_k^{(i)}=
\begin{cases}
\bepsilon_k^{(i)}, & k\in\mathcal{S},\\
\bepsilon_k^{\mathrm{grp}}, & k\notin\mathcal{S},
\end{cases}
\qquad
\bepsilon_k^{(i)},\,\bepsilon_k^{\mathrm{grp}}
\sim \mathcal{N}(\mathbf{0},\mathbf{I}),
\label{eq:shared-noise}
\end{equation}
where \(\{\bepsilon_k^{(i)}\}_{i=1}^N\) are independent across trajectories, while \(\bepsilon_k^{\mathrm{grp}}\) is shared by all trajectories in the same group at step \(k\).
Only the selected steps \(k\in\mathcal{S}\) contribute the gradients, yielding
\begin{equation}
\nabla_\theta \mathcal{L}_{\mathrm{stoc}}^{\mathrm{SubGRPO}}
=
\frac{K-1}{M}\cdot \frac{1}{N}\sum_{i=1}^N
\sum_{k\in\mathcal S}
A_i\,
\nabla_\theta
\log \pi_\theta^{(k)}
\!(
\hat{\bx}_{t_k}^{(i)} \mid \hat{\bx}_{t_{k-1}}^{(i)}
).
\label{eq:subgrpo}
\end{equation}
Under the same gradient sample budget, SubGRPO can be viewed as a Rao--Blackwellized variant of the corresponding independent-noise estimator under mild assumptions \citep{raolackwellisation}. Therefore, its gradient estimator typically has a smaller variance.

\noindent\textbf{Total objective.}
Combining Eqs.~\eqref{eq:det_grad}, \eqref{eq:subgrpo}, and \eqref{eq:subinterval-generator-gradient}, the generator in the second stage is updated by descending along:
\begin{equation}
\nabla_\theta \mathcal{L}_{\mathrm{total}}
=
\nabla_\theta \mathcal{L}_{\mathrm{AC\text{-}DMD}}
-
\beta\bigl(\nabla_\theta \mathcal{L}_{\mathrm{stoc}}^{\mathrm{SubGRPO}} + \nabla_\theta \mathcal{L}_{\mathrm{det}}\bigr).
\label{eq:total_loss}
\end{equation}

\section{Experiments}\label{sec:experiments}

\subsection{Implementation Details}\label{sec:implementation-details}

\textbf{Models.} Our experiments are conducted on open-source state-of-the-art (SOTA) text-to-image diffusion models: Stable Diffusion 3-Medium
(\texttt{SD3-M})~\citep{sd3}, Stable Diffusion 3.5-Medium
(\texttt{SD3.5-M})~\citep{sd35} and FLUX.2 4B~\citep{flux-2}. We use $512^2$ as the default resolution unless otherwise specified.

\textbf{Rewards.} For \texttt{SD3-M}, following prior work~\citep{dmdr, tdm}, we train with
HPSv2~\citep{hpsv2} and CLIPScore~\citep{clipscore} rewards on prompts from \texttt{t2i-2M}~\citep{t2i-2M}, and evaluate on
prompts sampled from ShareGPT-4o-Image~\citep{sharegpt4o-image}, reporting CLIPScore, Aesthetic Score~\citep{laion5b},
PickScore~\citep{pickscore}, and HPSv2. Besides, we further validate on the non-differentiable GenEval reward~\citep{geneval} for \texttt{SD3.5-M}.
For FLUX.2 4B, we train with HPSv2, CLIPScore, PickScore, and GenEval rewards, and additionally evaluate on OCR Score, Aesthetic Score, GenEval2~\citep{kamath2025geneval2addressingbenchmark},
ImageReward~\citep{xu2023imagereward}, and HPSv3~\citep{hpsv3} for a thorough assessment.

\textbf{Training.} We finetune the generator initialized from its corresponding pre-trained teacher without CFG~\citep{ho2022classifierfreediffusionguidance} using LoRA~\citep{lora} ($\alpha=32, r=64$) and adopt CPS~\citep{wang2025coefficients} with $\eta=0.9$ (see App.~\ref{app:discussion_sampling} for more discussion). For the cold start stage, we adopt $1.5{\times}10^3$ training iterations for the generator and $\gamma=0.01$. In the second stage, we use $1{\times}10^3$ iterations, and each consists of $48$ groups with a group size of $24$ for \texttt{SD3-M} and \texttt{SD3.5-M}, and $64$ groups with the same group size for FLUX.2 4B. All the experiments are conducted on $8$ or $16$ NVIDIA H20 GPUs. More details can be found in App.~\ref{app:more_details}.

\textbf{Baselines.} We compare our method against SOTA few-step RL approaches, including GDMD~\citep{dong2026guidingdistributionmatchingdistillation}, DMDR~\citep{dmdr}, $R_\text{dm}$~\citep{fan2026rtextdmreconceptualizingdistributionmatching},
TDM-R1~\citep{luo2026tdmr1reinforcingfewstepdiffusion}, and Hyper-SD~\citep{ren2024hyper}. To cover a wider range of baselines, we
further include foundational multi-step base models~\citep{flux-2, imageteam2025zimageefficientimagegeneration, sd35}, RL-only
approaches~\citep{xu2023imagereward}, and few-step distillation methods~\citep{lcm,dmd2,chadebec2025flash,tdm} as baselines. We reproduce the
closed-source $R_\text{dm}$ and evaluate the open-source baselines with their official checkpoints, while the remaining results are
directly taken from GDMD~\citep{dong2026guidingdistributionmatchingdistillation} and TDM-R1~\citep{luo2026tdmr1reinforcingfewstepdiffusion}.

\begin{table}[!ht]\setlength{\tabcolsep}{2pt}
 \renewcommand{\arraystretch}{1.1}
  \centering
  \caption{Quantitative comparison with SOTA approaches on \texttt{SD3-M} with $1024^2$ resolution. We highlight the best and second-best scores in bold and underlined formats, respectively. Visual results can be found in Fig.~\ref{fig:sd3_compare}} 
  \resizebox{0.9\linewidth}{!}{
  \newcommand{\dt}[1]{{\color{gray}\scriptsize #1}}                                                    
\begin{tabu}{l c cc cc cc cc cc}
\toprule
\multirow{2}{*}{\textbf{Method}} & \multirow{2}{*}{\textbf{NFE}} & \multicolumn{2}{c}{\textbf{CLIPScore}} & \multicolumn{2}{c}{\textbf{Aesthetic}} & \multicolumn{2}{c}{\textbf{PickScore}} &
\multicolumn{2}{c}{\textbf{HPSv2}} & \multicolumn{2}{c}{\textbf{ImageReward}} \\
\cmidrule(lr){3-4} \cmidrule(lr){5-6} \cmidrule(lr){7-8} \cmidrule(lr){9-10} \cmidrule(lr){11-12}
& & Score$\uparrow$ & \dt{$\Delta$($\uparrow$)} & Score$\uparrow$ & \dt{$\Delta$($\uparrow$)} &
Score$\uparrow$ & \dt{$\Delta$($\uparrow$)} & Score$\uparrow$ & \dt{$\Delta$($\uparrow$)} &
Score$\uparrow$ & \dt{$\Delta$($\uparrow$)} \\
\midrule
\texttt{SD3-M} (w/o CFG)~\citep{esser2024scaling}       & 50  & 0.2535 & \dt{$-$0.0401} & 5.2840 &
\dt{$-$0.2871} & 20.5660 & \dt{$-$1.7576} & 0.2010 & \dt{$-$0.0800} & $-$0.3762 & \dt{$-$1.4521} \\  
\texttt{SD3-M} (w/ CFG)~\citep{esser2024scaling}        & 100 & 0.2936 & \dt{---}       & 5.5711 &
\dt{---}       & 22.3236 & \dt{---}       & 0.2810 & \dt{---}       & 1.0759    & \dt{---}       \\  
\midrule
\multicolumn{12}{l}{\cellcolor[gray]{0.92}\textbf{\textit{Few-Step Distillation}}} \\
\midrule
DMD~\citep{dmd}                        & 4   & 0.2861 & \dt{$-$0.0075} & 5.5598 &
\dt{$-$0.0113} & 21.6216 & \dt{$-$0.7020} & 0.2891 & \dt{$+$0.0081}   & 0.9704    & \dt{$-$0.1055} \\
DMD2~\citep{dmd2}                  & 4   & 0.2914 & \dt{$-$0.0022} & 5.7704 &
\dt{$+$0.1993}   & 22.1442 & \dt{$-$0.1794} & 0.2951 & \dt{$+$0.0141}   & 1.1689    & \dt{$+$0.0930}   \\
Flash Diffusion~\citep{chadebec2025flash}               & 4   & 0.2864 & \dt{$-$0.0072} & 5.5236 &
\dt{$-$0.0475} & 22.0558 & \dt{$-$0.2678} & 0.2636 & \dt{$-$0.0174} & 0.8752    & \dt{$-$0.2007} \\
TDM~\citep{tdm}                             & 4   & 0.2848 & \dt{$-$0.0088} & 5.7070 &
\dt{$+$0.1359}   & 22.1806 & \dt{$-$0.1430} & 0.2940 & \dt{$+$0.0130}   & 1.0932    & \dt{$+$0.0173}   \\
\midrule
\multicolumn{12}{l}{\cellcolor[gray]{0.92}\textbf{\textit{Few-Step RL}}} \\
\midrule
Hyper-SD (w/ CFG)~\citep{ren2024hyper}                  & 8   & 0.2827 & \dt{$-$0.0109} & 5.5715 &
\dt{$+$0.0004}   & 21.2576 & \dt{$-$1.0660} & 0.2608 & \dt{$-$0.0202} & 0.6562    & \dt{$-$0.4197} \\
DMDR~\citep{dmdr}            & 4   & 0.2901 & \dt{$-$0.0035} & 5.5123 &
\dt{$-$0.0588} & 21.6528 & \dt{$-$0.6708} & 0.2931 & \dt{$+$0.0121}   & 1.0120    & \dt{$-$0.0639} \\
GDMD~\citep{dong2026guidingdistributionmatchingdistillation}    & 4   & 0.2930 & \dt{$-$0.0006} &
5.8728 & \dt{$+$0.3017} & 22.4614 & \dt{$+$0.1378} & \underline{0.3076} & \dt{$\underline{+0.0266}$} &
1.2702 & \dt{$+$0.1943} \\
$R_\text{dm}$~\citep{fan2026rtextdmreconceptualizingdistributionmatching} & 4 & \underline{0.2936} & \dt{$\underline{+0.0000}$} & \underline{5.8769} & \dt{$\underline{+0.3058}$} & \underline{22.5783} & \dt{$\underline{+0.2547}$} & 0.2957 & \dt{$+$0.0147} & \underline{1.2897} & \dt{$\underline{+0.2138}$} \\
\rowcolor{colSky!30} RTDMD (Ours) & 4 & \textbf{0.3161} & \dt{$\mathbf{+0.0225}$} & \textbf{5.9642} & \dt{$\mathbf{+0.3931}$} & \textbf{22.8593} & \dt{$\mathbf{+0.5357}$} & \textbf{0.3211} & \dt{$\mathbf{+0.0401}$} & \textbf{1.3024} & \dt{$\mathbf{+0.2265}$}\\
\bottomrule
\end{tabu}
}
    \label{tab:sd3}
\end{table}
\subsection{Performance Analysis}\label{sec:performance_analysis}
\noindent\textbf{Comparison with baselines.} We first evaluate our RTDMD on \texttt{SD3-M} with 4-step generation and report results in Tab.~\ref{tab:sd3}. RTDMD achieves the best performance across all five evaluation metrics, establishing a new state of the art for few-step
generation. Specifically, we attain a CLIPScore~\citep{clipscore} of 0.3161, a PickScore~\citep{pickscore} of 22.86, and an HPSv2~\citep{hpsv2} of 0.3211, outperforming the strongest prior methods $R_\text{dm}$~\citep{fan2026rtextdmreconceptualizingdistributionmatching} and GDMD~\citep{dong2026guidingdistributionmatchingdistillation} by a large margin. Notably, these
gains extend beyond training-time rewards to unseen evaluation metrics: our model reaches an Aesthetic Score of 5.9642 and an ImageReward of 1.3024, surpassing the 100-NFE teacher with CFG~\citep{ho2022classifierfreediffusionguidance} by substantial margins (+0.39 and
+0.23, respectively) while using $25\times$ fewer NFE. We additionally validate our framework with the non-differentiable GenEval~\citep{geneval} reward on \texttt{SD3.5-M} (see Tab.~\ref{tab:sd35}), where our approach generalizes effectively.

\begin{table}[!ht]\setlength{\tabcolsep}{2pt}
 \renewcommand{\arraystretch}{1.1}
  \centering
  \caption{Performance on larger-scale diffusion models. We evaluate HPSv3, OCR, and GenEval following \citet{flowgrpo} and \citet{xue2025dancegrpo}, and use the official setting for GenEval2. The remaining metrics are evaluated on
  DrawBench~\citep{schuhmann2022aesthetics} following \citet{zheng2026diffusionnft}.} 
  \resizebox{1.0\linewidth}{!}{
  \begin{tabu}{l|ccccccccc}
\toprule
\makecell{\textbf{Model}} &
\makecell{\textbf{ImageReward}} &
\makecell{\textbf{CLIPScore}} &
\makecell{\textbf{Aesthetic}} &
\makecell{\textbf{PickScore}} &
\makecell{\textbf{HPSv2}} &
\makecell{\textbf{HPSv3}} &
\makecell{\textbf{GenEval}} &
\makecell{\textbf{GenEval2}} &
\makecell{\textbf{OCR}} \\
\midrule

\multicolumn{10}{l}{\cellcolor[gray]{0.92}\textbf{\textit{Diffusion Models}}} \\
\midrule
FLUX.2 4B~\citep{flux-2} & 0.8538 & 0.2834 & 5.3333 & 22.3938 & 0.2771 & 11.7025 & 0.7631 & 0.2207 & 0.6133 \\
FLUX.2 9B~\citep{flux-2} & 1.0021 & \underline{0.2962} & 5.2030 & 22.6382 & 0.2800 & 11.6883 & 0.7568 & 0.3557 & 0.7432 \\
Z-Image 6B~\citep{zimage} & 0.7841 & 0.2841 & 5.2488 & 22.2118 & 0.2714 & 10.0857 & 0.6563 & 0.3012 & 0.7373 \\

\midrule
\multicolumn{10}{l}{\cellcolor[gray]{0.92}\textbf{\textit{Few-Step Diffusion Models (4 NFE)}}} \\
\midrule
Z-Image-Turbo 6B~\citep{zimage} & 0.9696 & 0.2764 & 5.2894 & 22.7994 & 0.2954 & 12.9136 & 0.7562 & 0.3530 & 0.7539 \\
FLUX.2 4B~\citep{flux-2} & 1.0506 & 0.2864 & 5.2658 & 22.7370 & 0.2890 & 12.9295 & 0.7722 & 0.2403 & 0.6375 \\
FLUX.2 9B~\citep{flux-2} & \underline{1.1998} & 0.2919 & \underline{5.3730} & \underline{23.0178} & 0.2991 & 13.2955 & \underline{0.7814} & \underline{0.3570} & \underline{0.7566} \\
Z-Image 6B w/ TDM-R1~\citep{luo2026tdmr1reinforcingfewstepdiffusion} & 1.1543 & 0.2836 & 5.2450 & 22.8202 & \underline{0.3064} & \underline{13.4349} & 0.7737 & \textbf{0.4073} & \textbf{0.7665} \\
\rowcolor{colSky!30}FLUX.2 4B w/ RTDMD (Ours) & \textbf{1.3712} & \textbf{0.3219} & \textbf{5.7746} & \textbf{23.9642} & \textbf{0.3516} & \textbf{15.5772} & \textbf{0.9046} & 0.2755 & 0.6858 \\
\bottomrule
\end{tabu}
}
    \label{tab:flux}
\end{table}

\begin{figure}[!ht]
   \centering
   \setlength{\abovecaptionskip}{0.2cm}
   \begin{minipage}[b]{0.8\linewidth}
       \begin{spacing}{0.7}
       {\tiny Text prompt: \textit{``Lego Arnold Schwarzenegger.''}}
    \vspace{0.05in}         \end{spacing}
       \begin{minipage}[b]{\linewidth}
            \centering
            \begin{subfigure}[tp!]{\linewidth}
            \centering
            \includegraphics[width=\linewidth]{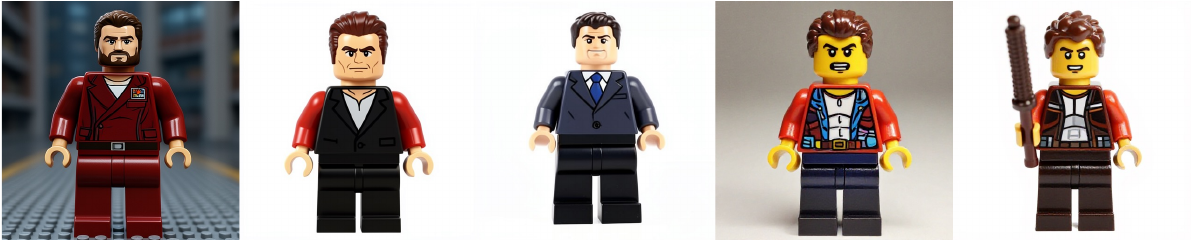}
            \end{subfigure}
       \end{minipage}
       \begin{spacing}{0.7}
       {\tiny Text prompt: \textit{``a green backpack and a pig.''}}
    \vspace{0.05in}         \end{spacing}
       \begin{minipage}[b]{\linewidth}
            \centering
            \begin{subfigure}[tp!]{\linewidth}
            \centering
            \includegraphics[width=\linewidth]{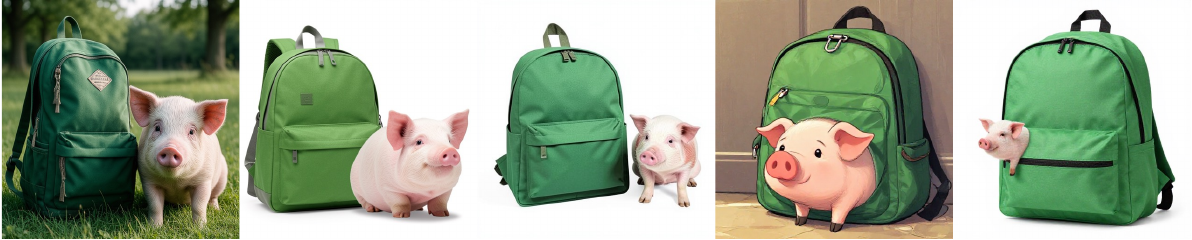}
            \end{subfigure}
       \end{minipage}
       \begin{spacing}{0.7}
       {\tiny Text prompt: \textit{``A vintage postcard with a faded, nostalgic look, featuring elegant cursive text that reads ``Wish You Were Here'' against a backdrop of a serene, old-world waterfront town with pastel buildings and a gentle, sunny sky.''}}
    \vspace{0.05in}         \end{spacing}
       \begin{minipage}[b]{\linewidth}
            \centering
            \begin{subfigure}[tp!]{\linewidth}
            \centering
            \includegraphics[width=\linewidth]{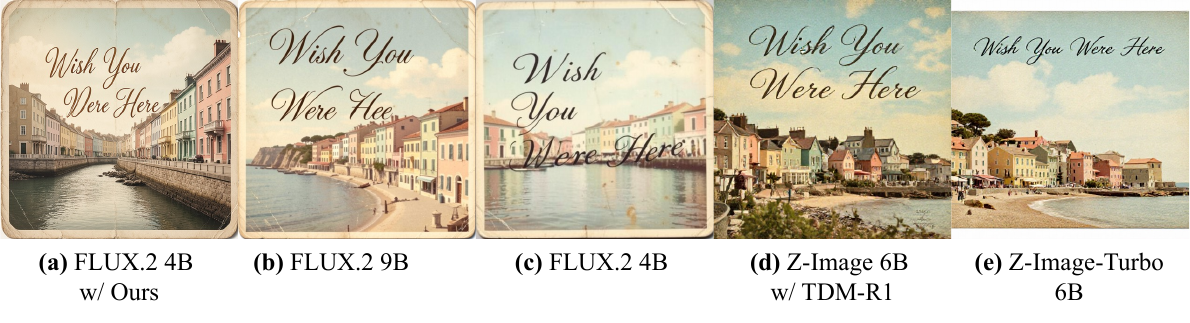}
            \end{subfigure}
       \end{minipage}
   \end{minipage}
   \vspace{-0.05in}
   \caption{Qualitative comparison for few-step diffusion models (4 NFE). Using identical noise inputs, our method outperforms others in both quality and prompt alignment, showing strong performance.}
   \label{fig:flux_compare_main}
\end{figure}

\noindent\textbf{Scaling to more advanced models.} In Tab.~\ref{tab:flux}, we further apply our RTDMD to FLUX.2 4B~\citep{flux-2}, a SOTA transformer-based flow model. Our approach sets new best results in seven out of 9 metrics with substantial absolute gains over the 50-step FLUX.2 4B baseline. Notably, our 4B model surpasses the considerably larger FLUX.2 9B (50-step) on the
majority of metrics, demonstrating that RL-guided distillation can effectively close the quality gap introduced by model scale reduction. While Z-Image 6B~\citep{zimage} with TDM-R1~\citep{luo2026tdmr1reinforcingfewstepdiffusion} achieves higher absolute scores on GenEval2~\citep{kamath2025geneval2addressingbenchmark} and OCR~\citep{flowgrpo}, this advantage stems
primarily from the Z-Image base model itself being inherently stronger on these two benchmarks; in contrast, the relative improvement brought by
our method over its own baseline is more pronounced on OCR (+0.0483 \emph{vs.} +0.0126) and comparable on GenEval2. Qualitative comparisons in Fig.~\ref{fig:flux_compare_main} and Fig.~\ref{fig:flux_compare} further corroborate these findings.

\subsection{Ablation Studies}\label{sec:ablation}
In this subsection, we perform ablation using \texttt{SD3.5-M}~\citep{sd35} with CLIPScore~\citep{clipscore}, PickScore~\citep{pickscore}, and HPSv2~\citep{hpsv2} as training rewards on DrawBench~\citep{schuhmann2022aesthetics} following ~\citet{zheng2026diffusionnft}. Other settings are the same as Sec.~\ref{sec:implementation-details}. Visualization results for ablation are also presented in App.~\ref{app:vis_ablation}.

\begin{table}[!ht]\setlength{\tabcolsep}{1pt}
  \centering
  \begin{minipage}[b]{0.495\linewidth}
  \renewcommand{\arraystretch}{1.1}
    \caption{Ablation results for different distribution matching strategies in our two-stage training. All methods adopt the CPS scheduler~\citep{wang2025coefficients} with $\eta=0.9$. \txtBlue{Blue row} denotes the default setting of our RTDMD. Results of DMD2~\citep{dmd2} and A-DMD with various values of $\eta$ are shown in App.~\ref{app:discussion_sampling}.}
      \resizebox{\linewidth}{!}{
      \newcommand{\dt}[1]{{\color{gray}\scriptsize #1}}
\begin{tabu}{l cc cc cc}
\toprule
\multirow{2}{*}{\textbf{Method}} & \multicolumn{2}{c}{\textbf{CLIPScore}} & \multicolumn{2}{c}{\textbf{PickScore}} &
\multicolumn{2}{c}{\textbf{HPSv2}} \\
\cmidrule(lr){2-3} \cmidrule(lr){4-5} \cmidrule(lr){6-7}
& Score$\uparrow$ & \dt{$\Delta$($\uparrow$)} &
Score$\uparrow$ & \dt{$\Delta$($\uparrow$)} & Score$\uparrow$ & \dt{$\Delta$($\uparrow$)} \\
\midrule
A-DMD & 0.2753 & \dt{---} & 22.0116 & \dt{---} & 0.3188 & \dt{---} \\
AC-DMD ($\gamma\!=\!0.1$) & \textbf{0.2846} & \dt{$\mathbf{+0.0093}$} & \underline{23.2105} & \dt{$\underline{+1.1989}$} & \underline{0.3362} & \dt{$\underline{+0.0174}$} \\
\rowcolor{colSky!30}AC-DMD ($\gamma\!=\!0.01$) & \underline{0.2843} & \dt{$\underline{+0.0090}$} & \textbf{23.5690} & \dt{$\mathbf{+1.5574}$} & \textbf{0.3456} & \dt{$\mathbf{+0.0268}$} \\
AC-DMD ($\gamma\!=\!0.005$) & 0.2837 & \dt{$+$0.0084} & 22.8546 & \dt{$+$0.8430} & 0.3324 & \dt{$+$0.0136} \\
AC-DMD ($\gamma\!=\!0.001$) & 0.2831 & \dt{$+$0.0078} & 22.2378 & \dt{$+$0.2262} & 0.3294 & \dt{$+$0.0106} \\
\bottomrule
\end{tabu}
}
        \label{tab:ac-dmd}
    \end{minipage}
    \begin{minipage}[b]{0.495\linewidth}
    \renewcommand{\arraystretch}{1.2}
    \caption{Ablation results for strategies applied after cold start. ``w/ $\mathcal{L}_{\mathrm{det}}$'' means we adopt Eq.~\eqref{eq:det_grad} to optimize the generator. ``w/ GRPO'' signifies that we use naive GRPO to optimize the first $K-1$ stochastic steps. \txtBlue{Blue row} denotes the default setting of our RTDMD.}
      \resizebox{\linewidth}{!}{
      \newcommand{\dt}[1]{{\color{gray}\scriptsize #1}}
\begin{tabu}{l cc cc cc}
\toprule
\multirow{2}{*}{\textbf{Method}} & \multicolumn{2}{c}{\textbf{CLIPScore}} & \multicolumn{2}{c}{\textbf{PickScore}} &
\multicolumn{2}{c}{\textbf{HPSv2}} \\
\cmidrule(lr){2-3} \cmidrule(lr){4-5} \cmidrule(lr){6-7}
& Score$\uparrow$ & \dt{$\Delta$($\uparrow$)} &
Score$\uparrow$ & \dt{$\Delta$($\uparrow$)} & Score$\uparrow$ & \dt{$\Delta$($\uparrow$)} \\
\midrule
Cold Start & 0.2822 & \dt{$+$0.0063} & 22.5056 & \dt{$-$0.8008} & 0.2942 & \dt{$-$0.0326} \\
$\quad \text{w/}$ GRPO & 0.2759 & \dt{---} & 23.3064 & \dt{---} & 0.3268 & \dt{---} \\
$\quad \text{w/}$ SubGRPO ($M\!=\!1$) & 0.2798 & \dt{$+$0.0039} & 23.3272 & \dt{$+$0.0208} & 0.3332 & \dt{$+$0.0064} \\
$\quad\quad\text{w/} \mathcal{L}_{\mathrm{det}}$ & 0.2802 & \dt{$+$0.0043} & 23.4156 & \dt{$+$0.1092} & 0.3318 & \dt{$+$0.0050} \\
$\quad \text{w/}$ SubGRPO ($M\!=\!2$) & \underline{0.2839} & \dt{$\underline{+0.0080}$} & \underline{23.4520} & \dt{$\underline{+0.1456}$} & \textbf{0.3459} & \dt{$\mathbf{+0.0191}$} \\
\rowcolor{colSky!30}$\quad\quad \text{w/} \mathcal{L}_{\mathrm{det}}$  & \textbf{0.2843} & \dt{$\mathbf{+0.0084}$} & \textbf{23.5690} & \dt{$\mathbf{+0.2626}$} & \underline{0.3456} & \dt{$\underline{+0.0188}$} \\
\bottomrule
\end{tabu}
}
        \label{tab:hybrid-policy-gradient}
    \end{minipage}\hfill
\end{table}
\begin{figure}[!ht]
   \centering
     \includegraphics[width=\textwidth]{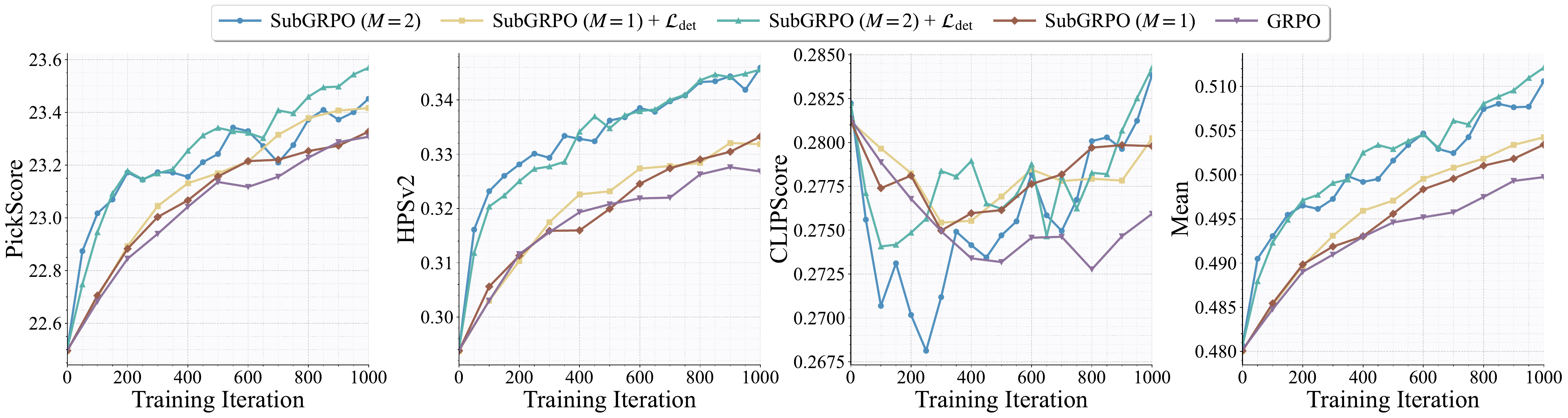}
     \vspace{-0.2in}
     \caption{Evaluation curves when reinforcing the few-step generator. ``Mean'' denotes the average score of  normalized PickScore, HPSv2, and CLIPScore.}
    \label{fig:reward_curve_ablation}
\end{figure}

\noindent\textbf{Effect of AC-DMD.} In A-DMD, the fake score model is trained on noisy intermediate latents rather than clean samples, making denoising score matching less accurate under limited updates. As shown in
Tab.~\ref{tab:ac-dmd}, the consistency loss
$\mathcal{L}_{\mathrm{cons}}^{(k)}$ improves fake score estimation by enforcing coherent predictions across neighboring timesteps. Notably, the improvement is consistent across all tested values of $\gamma$, confirming the robustness of the regularizer. The best result at $\gamma\!=\!0.01$ yields an additional +1.58 PickScore
and +0.027 HPSv2 over A-DMD.

\noindent\textbf{Effect of SubGRPO.}
As shown in Tab.~\ref{tab:hybrid-policy-gradient} and Fig.~\ref{fig:reward_curve_ablation}, replacing naive GRPO with SubGRPO consistently improves all metrics. Naive GRPO samples independent noise at every stochastic step, so reward differences across trajectories conflate
contributions from all steps, leading to a noisy gradient signal. SubGRPO instead shares the injected noise at non-selected steps ($k\notin\mathcal{S}$) across the group (Eq.~\eqref{eq:shared-noise}), attributing reward variation primarily to
the selected subset $\mathcal{S}$. This shared-noise design reduces gradient variance by isolating the effect of selected steps. Increasing the subset size from $M\!=\!1$ to $M\!=\!2$ brings further gains, with PickScore improving from 23.33 to 23.45 and HPSv2 from 0.3332 to 0.3459. This reflects a trade-off between per-step variance reduction (fewer selected steps yield tighter control) and optimization coverage (more selected steps provide gradient signal to a larger portion of the trajectory); in our setting, $M\!=\!2$ strikes a favorable balance.

\noindent\textbf{Effect of hybrid policy gradient.}
Existing GRPO-based methods for diffusion models~\citep{flowgrpo,xue2025dancegrpo} only optimize the stochastic denoising steps via the GRPO-estimator, neglecting the final deterministic transition entirely. However, our hybrid policy gradient (Eq.~\eqref{eq:hybrid-policy-gradient}) combines SubGRPO for
stochastic intermediate steps with a pathwise gradient through this deterministic final step. Tab.~\ref{tab:hybrid-policy-gradient} shows that adding this term consistently improves performance: under $M\!=\!2$, PickScore increases from 23.45 to 23.57 and CLIPScore from 0.2839 to 0.2843, confirming that the deterministic final mapping carries a meaningful optimization signal that is otherwise left unexploited.

\section{Conclusions and Limitations}\label{sec:conclusion_and_limitations}
We present RTDMD, a two-stage framework for few-step image generation that unifies distribution matching distillation with reward-guided reinforcement learning. The first stage,
AC-DMD employs ambient distribution matching and augments fake score training with a consistency regularizer to stabilize estimation. The second stage minimizes the KL divergence to a reward-tilted teacher distribution, which naturally decomposes into a distribution matching term and a reward maximization term. For the latter, we derive a hybrid policy gradient that combines SubGRPO for stochastic intermediate steps with direct reward backpropagation through the deterministic final step. Comprehensive experiments demonstrate the SOTA performance of our method across baselines. In terms of limitations, this work focuses on text-to-image generation; extending the framework to video synthesis and image editing, where reward design and temporal consistency pose additional challenges, is left for future work.

\bibliography{citations}

\clearpage

\appendix

\clearpage

\section{Related Work}\label{sec:related_work}

\noindent\textbf{Few-step distillation.} To alleviate the iterative sampling cost of diffusion and flow-based models~\citep{ddpm,ldm,sd3,flow_matching}, a growing line of work distills multi-step teachers into few-step students, broadly falling into \emph{trajectory-based} and
\emph{distribution-based} approaches. The former, including progressive distillation~\citep{salimans2022progressive,ren2024hyper,lin2024sdxl,chadebec2025flash} and consistency
models (CM)~\citep{song2023consistency,lcm,kimctm,wang2024phased}, replicate or enforce self-consistency along the teacher's trajectory, while the latter align output distributions via
GAN-based~\citep{Wang2023DiffusionGAN,sauer2024adversarial,sauer2024fast} or score/VSD variants~\citep{wang2023prolificdreamer,Luo2023diffinstruct}. Among these, Distribution Matching Distillation (DMD)~\citep{dmd,dmd2} has become a foundation
work for high-fidelity few-step generation, with follow-ups refining it from different angles: Flash-DMD~\citep{flash-dmd} uses a timestep-aware pixel-GAN; TDM~\citep{tdm} fuses trajectory and distribution matching without GAN; Decoupled
DMD~\citep{decoupled-dmd} recasts classifier-free guidance (CFG)~\citep{ho2022classifierfreediffusionguidance} as the generative driver and DMD as a regularizer; and PhasedDMD~\citep{fan2026phaseddmdfewstepdistribution} combines phase-wise distillation
with Mixture-of-Experts (MoE)~\citep{jiang2024mixtralexperts} to ease learning.

\noindent\textbf{Reinforcement learning for diffusion models.} Reinforcement learning (RL) has emerged as a powerful tool for aligning diffusion and flow-based models with human preferences. One line directly backpropagates reward gradients, as in ReFL~\citep{xu2023imagereward} on one-step predictions and DRaFT~\citep{clark2023directly} on multi-step samples via truncated backpropagation. Another line formulates denoising as a multi-step decision process: DDPO~\citep{ddpo} and DPOK~\citep{fan2023dpok} derive tractable Gaussian likelihoods via an Euler–Maruyama discretization, enabling PPO-style~\citep{schulman2017proximal} optimization. Flow-GRPO~\citep{flowgrpo} and DanceGRPO~\citep{xue2025dancegrpo} further combine this formulation with GRPO~\citep{deepseekmath}, thereby simplifying PPO with a group-mean baseline. A third line optimizes the forward process: \citet{lee2023aligning} and \citet{fan2025online} use
reward-weighted denoising losses, Diffusion-DPO~\citep{wallace2024diffusion} adapts DPO~\citep{rafailov2024directpreferenceoptimizationlanguage} without rollouts, and FMPG~\citep{mcallister2025flow} and AWM~\citep{xue2025advantage} use the ELBO as a likelihood proxy, while
DiffusionNFT~\citep{zheng2026diffusionnft} reinterprets GRPO as an NFT-style~\citep{chen2026nftbridgingsupervisedlearning} forward-process variant.

\textbf{Reinforcement learning for few-step generative models.} Beyond applying RL to multi-step teachers, a growing body of work integrates RL with few-step distillation to combine efficiency with preference alignment. Early efforts treat RL as an independent post-training stage on top of distilled models: Hyper-SD~\citep{ren2024hyper} augments trajectory-segmented consistency distillation with human feedback learning, while PSO~\citep{miao2025tuningtimestepdistilleddiffusionmodel} fine-tunes timestep-distilled models with pairwise preference data to avoid costly re-distillation.
Diff-Instruct++~\citep{luo2025diffinstructtrainingonesteptexttoimage} further reveals a theoretical link between CFG-based distillation and RL, formulating one-step generator alignment as KL-regularized reward maximization. More recent work~\citep{luo2026tdmr1reinforcingfewstepdiffusion, dmdr}
moves toward joint optimization: DMDR~\citep{dmdr} unifies DMD and RL so that the two mutually regularize each other, and $R_{\text{dm}}$~\citep{fan2026rtextdmreconceptualizingdistributionmatching} re-conceptualizes distribution matching
itself as a reward, providing a principled bridge between DMD and GRPO-style RL. Besides these two works, GDMD~\citep{dong2026guidingdistributionmatchingdistillation} first combines NFT-style RL~\citep{zheng2026diffusionnft} with DMD.

\section{Derivation of the Step-\texorpdfstring{$k$}{k} Training Objective}
\label{app:k_step_objective}

\noindent\textbf{Learning the generator after $k$ sampling steps.}
Let $p_\theta^{(k)}$ denote the distribution of the generated state $\hat{\bx}_{t_k}$ after the first $k$ sampling steps.
For training step $k$, we define the student marginal at time $t \in [t_k,1]$ by diffusing $\hat{\bx}_{t_k}$ forward within the remaining interval.

Under the rectified schedule, the diffusion path is
\[
\bx_t = (1-t)\bx_0 + t\bepsilon,
\qquad
\bepsilon \sim \mathcal{N}(\mathbf{0},\mathbf{I}).
\]
By the Markov property of the Gaussian path, conditioning on an intermediate state $\bx_{t_k}$ yields
\[
\bx_t
=
\frac{1-t}{1-t_k}\bx_{t_k}
+
\frac{t-t_k}{1-t_k}\bepsilon,
\qquad t\in[t_k,1].
\]
Therefore, for generated samples $\hat{\bx}_{t_k}\sim p_\theta^{(k)}$, we define
\[
\bx_t^{(k)}
=
\alpha_k(t)\hat{\bx}_{t_k}+\sigma_k(t)\bepsilon,
\qquad
\alpha_k(t)=\frac{1-t}{1-t_k},
\qquad
\sigma_k(t)=\frac{t-t_k}{1-t_k},
\]
with conditional density
\[
q_k(\bx_t \mid \hat{\bx}_{t_k}, t)
=
\mathcal{N}\!\left(
\bx_t;\,
\alpha_k(t)\hat{\bx}_{t_k},\,
\sigma_k(t)^2\mathbf{I}
\right).
\]
The induced student marginal is
\[
p_{\theta,t}^{(k)}(\bx)
=
\int
q_k(\bx \mid \hat{\bx}_{t_k}, t)\,
p_\theta^{(k)}(\hat{\bx}_{t_k})\,
d\hat{\bx}_{t_k}.
\]

Following the reverse-KL formulation used in distribution matching distillation, we optimize
\[
\mathcal{L}_{\mathrm{gen}}^{(k)}(\theta)
=
\mathbb{E}_{t \sim \mathcal{U}[t_k,1]}
\left[
\lambda_k(t)\,
\mathrm{KL}\!\left(
p_{\theta,t}^{(k)} \,\|\, p_{\psi,t}
\right)
\right].
\]
Using the pathwise derivative of the reverse KL gives
\[
\nabla_\theta \mathcal{L}_{\mathrm{gen}}^{(k)}
=
\mathbb{E}_{t,\bepsilon,\hat{\bx}_{t_k}}
\left[
\lambda_k(t)\,
\left(
\nabla_{\bx}\log p_{\theta,t}^{(k)}(\bx_t^{(k)})
-
s_\psi(\bx_t^{(k)},t)
\right)^\top
\frac{\partial \bx_t^{(k)}}{\partial \theta}
\right].
\]
Since
\[
\bx_t^{(k)}=\alpha_k(t)\hat{\bx}_{t_k}+\sigma_k(t)\bepsilon,
\]
we have
\[
\frac{\partial \bx_t^{(k)}}{\partial \theta}
=
\alpha_k(t)\frac{\partial \hat{\bx}_{t_k}}{\partial \theta},
\]
and hence
\[
\nabla_\theta \mathcal{L}_{\mathrm{gen}}^{(k)}
=
\mathbb{E}_{t,\bepsilon,\hat{\bx}_{t_k}}
\left[
\lambda_k(t)\alpha_k(t)\,
\left(
\nabla_{\bx}\log p_{\theta,t}^{(k)}(\bx_t^{(k)})
-
s_\psi(\bx_t^{(k)},t)
\right)^\top
\frac{\partial \hat{\bx}_{t_k}}{\partial \theta}
\right].
\]

Because the score of $p_{\theta,t}^{(k)}$ is intractable, we approximate it with a learned fake score
\[
\nabla_{\bx}\log p_{\theta,t}^{(k)}(\bx)
\approx
s_\phi(\bx,t),
\]
which yields the practical score-difference gradient
\[
\nabla_\theta \mathcal{L}_{\mathrm{gen}}^{(k)}
\approx
-
\mathbb{E}_{t,\bepsilon,\hat{\bx}_{t_k}}
\left[
\lambda_k(t)\alpha_k(t)\,
\left(
s_\psi(\bx_t^{(k)},t)
-
s_\phi(\bx_t^{(k)},t)
\right)^\top
\frac{\partial \hat{\bx}_{t_k}}{\partial \theta}
\right].
\]

\noindent\textbf{Learning the fake score within $[t_k,1]$.}
Since clean samples $\bx_0$ are unavailable when $t_k>0$, the fake score must be trained using the generated endpoint $\hat{\bx}_{t_k}$ rather than a global clean target. Restating Eq.~\eqref{eq:ambient_score_matching}, the fake-score model $s_\phi$ is fit on the subinterval $[t_k,1]$ via the DSM objective
\[
\mathcal{L}_{\mathrm{fake}}^{(k)}(\phi)
=
\mathbb{E}_{t\sim\mathcal{U}[t_k,1],\,\hat{\bx}_{t_k},\,\bepsilon}
\left[
\omega_k(t)\,
\left\|
s_\phi(\bx_t^{(k)}, t)
-
\nabla_{\bx}\log q_t^{(k)}(\bx_t^{(k)} \mid \hat{\bx}_{t_k})
\right\|_2^2
\right],
\]
which is unbiased for the marginal score $\nabla_\bx\log p_{\theta,t}^{(k)}$ (App.~\ref{app:dsm_optimality}); reusing the full-interval objective with a mismatched conditioning would in general yield a biased estimator.

In practice, we parametrize $s_\phi$ via an $x$-prediction model $f_\phi$ that predicts $\hat{\bx}_{t_k}$. Under the Gaussian subinterval kernel
\[
q_t^{(k)}(\bx \mid \hat{\bx}_{t_k})
=
\mathcal{N}\!\left(
\bx;\,
\alpha_k(t)\hat{\bx}_{t_k},\,
\sigma_k(t)^2\mathbf{I}
\right),
\]
the score estimator follows from Tweedie's formula:
\[
s_\phi(\bx_t^{(k)},t)
=
-
\frac{
\bx_t^{(k)}-\alpha_k(t)\,f_\phi(\bx_t^{(k)},t)
}{
\sigma_k(t)^2
}.
\]

\section{Coefficient-Preserving Sampling Formula}\label{app:cps_formula}

Under the rectified-flow modeling, $\bx_t = (1-t)\bx_0 + t\bepsilon$ with $\bepsilon\sim\mathcal{N}(\mathbf{0},\mathbf{I})$, so $\alpha_t = 1-t$ and $\sigma_t = t$.
Given the current latent $\hat{\bx}_{t_{k-1}}$ at noise level $t_{k-1}$, the generator produces an $x$-prediction $\hat{\bx}_{\mathrm{pred}}^{(k)} = G_\theta(\hat{\bx}_{t_{k-1}}, t_{k-1})$ and a corresponding noise estimate
\begin{equation}
\hat{\bepsilon}^{(k)} = \frac{\hat{\bx}_{t_{k-1}} - (1-t_{k-1})\hat{\bx}_{\mathrm{pred}}^{(k)}}{t_{k-1}}.
\label{eq:noise_estimate}
\end{equation}

coefficient-preserving sampling (CPS)~\citep{wang2025coefficients} constructs the next latent $\hat{\bx}_{t_k}$ by enforcing that the signal and noise coefficients exactly match the scheduler at time $t_k$:
\begin{equation}
\hat{\bx}_{t_k}
=
(1-t_k)\,\hat{\bx}_{\mathrm{pred}}^{(k)}
+
t_k\cos\!\Bigl(\frac{\eta\pi}{2}\Bigr)\hat{\bepsilon}^{(k)}
+
t_k\sin\!\Bigl(\frac{\eta\pi}{2}\Bigr)\bepsilon_k,
\label{eq:cps_basic}
\end{equation}
where $\bepsilon_k\sim\mathcal{N}(\mathbf{0},\mathbf{I})$ is fresh noise and $\eta\in[0,1]$ controls stochasticity. Setting $\eta=0$ recovers the deterministic Euler ODE scheduler $\hat{\bx}_{t_k}=(1-t_k)\hat{\bx}_{\mathrm{pred}}^{(k)}+t_k\hat{\bepsilon}^{(k)}$, while $\eta=1$ fully replaces the old noise with fresh samples, recovering the consistency model (CM) scheduler~\citep{song2023consistency}. CPS thus provides a unified formulation that subsumes both common DMD~\citep{dmd2, dmd} sampling (\emph{i.e.}, Euler and CM) strategies as special cases.

Substituting $\hat{\bepsilon}^{(k)}$ from Eq.~\eqref{eq:noise_estimate} into Eq.~\eqref{eq:cps_basic} and collecting terms gives
\begin{align}
\hat{\bx}_{t_k}
&=
(1-t_k)\,\hat{\bx}_{\mathrm{pred}}^{(k)}
+
\frac{t_k}{t_{k-1}}\cos\!\Bigl(\frac{\eta\pi}{2}\Bigr)
\bigl[\hat{\bx}_{t_{k-1}}-(1-t_{k-1})\hat{\bx}_{\mathrm{pred}}^{(k)}\bigr]
+
t_k\sin\!\Bigl(\frac{\eta\pi}{2}\Bigr)\bepsilon_k \nonumber\\[4pt]
&=
\Bigl[
1-t_k+t_k\cos\!\Bigl(\frac{\eta\pi}{2}\Bigr)\Bigl(1-\frac{1}{t_{k-1}}\Bigr)
\Bigr]\hat{\bx}_{\mathrm{pred}}^{(k)}
+
\frac{t_k}{t_{k-1}}\cos\!\Bigl(\frac{\eta\pi}{2}\Bigr)\hat{\bx}_{t_{k-1}}
+
t_k\sin\!\Bigl(\frac{\eta\pi}{2}\Bigr)\bepsilon_k.
\label{eq:cps}
\end{align}
The transition defines a Gaussian policy with mean $\bmu_\theta^{(k)} = \hat{\bx}_{t_k}\big|_{\bepsilon_k=\mathbf{0}}$ and standard deviation $\sigma_k = t_k\sin(\eta\pi/2)$.

\section{Optimality of Ambient Denoising Score Matching}\label{app:dsm_optimality}
We show that the denoising score matching (DSM) objective in Eq.~\eqref{eq:ambient_score_matching} has the true student marginal score $\nabla_{\bx}\log p_{\theta,t}^{(k)}(\bx)$ as its unique optimal solution.
\begin{proposition}\label{prop:dsm_optimal}
Let $p_{\theta,t}^{(k)}(\bx) = \int q_t^{(k)}(\bx\mid \hat{\bx}_{t_k})\,p_\theta^{(k)}(\hat{\bx}_{t_k})\,d\hat{\bx}_{t_k}$ denote the student marginal. Then, over functions $s:\mathbb{R}^d\times[t_k,1]\to\mathbb{R}^d$,
\[
\argmin_{s}\;\mathbb{E}_{t\sim\mathcal{U}[t_k,1]}\,\mathbb{E}_{\hat{\bx}_{t_k}\sim p_\theta^{(k)},\,\bx\sim q_t^{(k)}(\cdot\mid\hat{\bx}_{t_k})}\left[\omega_k(t)\,\left\|s(\bx,t)-\nabla_{\bx}\log
q_t^{(k)}(\bx\mid\hat{\bx}_{t_k})\right\|^2\right]
\]
is achieved at $s^*(\bx,t) = \nabla_{\bx}\log p_{\theta,t}^{(k)}(\bx)$ for almost every $t\in[t_k,1]$.
\end{proposition}

\begin{proof}
Since $s$ is a function of both $(\bx,t)$, the outer expectation over $t$ decouples and the minimization can be carried out independently for each $t$. Indeed, the full objective can be written as
\[
\frac{1}{1-t_k}\int_{t_k}^1 \omega_k(t)\,\mathcal{J}_t\bigl(s(\cdot,t)\bigr)\, dt,
\qquad
\mathcal{J}_t(f) := \mathbb{E}_{\hat{\bx}_{t_k},\,\bx}\left[\|f(\bx) - \nabla_{\bx}\log q_t^{(k)}(\bx\mid\hat{\bx}_{t_k})\|^2\right].
\]
For each fixed $t\in[t_k,1]$, $\omega_k(t)>0$ is a strictly positive constant and does not affect the inner argmin, so it suffices to minimize $\mathcal{J}_t(f)$ over $f:\mathbb{R}^d\to\mathbb{R}^d$.

Since the joint density of $(\bx, \hat{\bx}_{t_k})$ is $p_\theta^{(k)}(\hat{\bx}_{t_k})\,q_t^{(k)}(\bx\mid\hat{\bx}_{t_k})$, we can write
\[
\mathcal{J}_t(f) = \int p_{\theta,t}^{(k)}(\bx)\,\mathbb{E}_{\hat{\bx}_{t_k}\sim p(\hat{\bx}_{t_k}\mid\bx)}\left[\|f(\bx) - \nabla_{\bx}\log
q_t^{(k)}(\bx\mid\hat{\bx}_{t_k})\|^2\right]d\bx,
\]
where $p(\hat{\bx}_{t_k}\mid\bx) = \frac{q_t^{(k)}(\bx\mid\hat{\bx}_{t_k})\,p_\theta^{(k)}(\hat{\bx}_{t_k})}{p_{\theta,t}^{(k)}(\bx)}$ is the posterior. For each $\bx$, the inner
expectation is minimized pointwise by choosing $f(\bx)$ equal to the conditional mean:
\[
f^*(\bx) = \mathbb{E}_{\hat{\bx}_{t_k}\sim p(\hat{\bx}_{t_k}\mid\bx)}\left[\nabla_{\bx}\log q_t^{(k)}(\bx\mid\hat{\bx}_{t_k})\right].
\]
It remains to show that this equals the marginal score. By Bayes' rule,
\begin{align}
\nabla_{\bx}\log p_{\theta,t}^{(k)}(\bx)
&= \nabla_{\bx}\log\int q_t^{(k)}(\bx\mid\hat{\bx}_{t_k})\,p_\theta^{(k)}(\hat{\bx}_{t_k})\,d\hat{\bx}_{t_k} \nonumber\\
&= \frac{\int \nabla_{\bx}\,q_t^{(k)}(\bx\mid\hat{\bx}_{t_k})\,p_\theta^{(k)}(\hat{\bx}_{t_k})\,d\hat{\bx}_{t_k}}{p_{\theta,t}^{(k)}(\bx)} \nonumber\\
&= \int \frac{q_t^{(k)}(\bx\mid\hat{\bx}_{t_k})\,p_\theta^{(k)}(\hat{\bx}_{t_k})}{p_{\theta,t}^{(k)}(\bx)}\,\nabla_{\bx}\log q_t^{(k)}(\bx\mid\hat{\bx}_{t_k})\,d\hat{\bx}_{t_k}
\nonumber\\
&= \mathbb{E}_{\hat{\bx}_{t_k}\sim p(\hat{\bx}_{t_k}\mid\bx)}\left[\nabla_{\bx}\log q_t^{(k)}(\bx\mid\hat{\bx}_{t_k})\right] = f^*(\bx). \nonumber
\end{align}
Hence the pointwise optimum $s^*(\bx,t) = \nabla_{\bx}\log p_{\theta,t}^{(k)}(\bx)$ holds for almost every $t\in[t_k,1]$.
\end{proof}

\section{Self-Consistency of the Optimal Fake Score}\label{app:consistency_proof}                                                                                                      

We prove that the optimal fake score model satisfies the self-consistency property invoked in Sec.~\ref{sec:ac-dm}.                                                                    
                                                        
\begin{proposition}[Self-consistency of optimal denoiser]\label{prop:consistency}
Let $\hat{\bx}_\phi(\bx_t, t)$ denote a score model in $x$-prediction form trained under the flow matching framework. If $\hat{\bx}_\phi$ is optimal, i.e.,
\[
\hat{\bx}_\phi(\bx_t, t) = \mathbb{E}[\bx_0 \mid \bx_t], \qquad \forall\, t,
\]
then for any $t'' < t'$:
\[
\hat{\bx}_\phi(\bx_{t'}, t') = \mathbb{E}_{\tilde{\bx}_{t''} \sim p(\bx_{t''} \mid \bx_{t'})}\bigl[\hat{\bx}_\phi(\tilde{\bx}_{t''}, t'')\bigr].
\]
\end{proposition}

\begin{proof}
Under the forward process, $\bx_t = \alpha_t\bx_0 + \sigma_t\bepsilon$ with $\bx_0 \sim p_0$, $\bepsilon \sim \mathcal{N}(\mathbf{0}, \mathbf{I})$, and $(\alpha_t, \sigma_t)$ specifying the noise schedule (\emph{e.g.}, $\alpha_t = 1-t$, $\sigma_t = t$ for the rectified-flow schedule). For $t'' < t'$, the Gaussian interpolation implies that $(\bx_0,\, \bx_{t''},\, \bx_{t'})$ forms a Markov chain in the increasing-noise direction, yielding
\[
\mathbb{E}[\bx_0 \mid \bx_{t''},\, \bx_{t'}] = \mathbb{E}[\bx_0 \mid \bx_{t''}].
\]
Starting from the right-hand side and applying optimality at $t''$:
\begin{align}
\mathbb{E}_{\tilde{\bx}_{t''} \sim p(\bx_{t''} \mid \bx_{t'})}\bigl[\hat{\bx}_\phi(\tilde{\bx}_{t''}, t'')\bigr]
&= \mathbb{E}\bigl[\mathbb{E}[\bx_0 \mid \bx_{t''}]\;\big|\;\bx_{t'}\bigr] \nonumber\\                                                                                                 
&= \mathbb{E}\bigl[\mathbb{E}[\bx_0 \mid \bx_{t''},\, \bx_{t'}]\;\big|\;\bx_{t'}\bigr] \nonumber\\
&= \mathbb{E}[\bx_0 \mid \bx_{t'}] = \hat{\bx}_\phi(\bx_{t'}, t'). \nonumber                                                                                                           
\end{align}                                               
The second equality uses the Markov property; the third is the tower rule.
\end{proof}

The consistency regularizer $\mathcal{L}_{\mathrm{cons}}^{(k)}$ (Eq.~\eqref{eq:consistency_loss}) penalizes deviations from this property, constraining the fake score model toward the
optimal solution even as the generator distribution shifts during training.

\section{Practical Consistency Loss Estimator}\label{app:consistency_practice}
Computing $\mathcal{L}_{\mathrm{cons}}^{(k)}$ (Eq.~\eqref{eq:consistency_loss}) requires evaluating $\mathbb{E}_{\tilde{\bx}_{t''}}[\hat{\bx}_\phi(\tilde{\bx}_{t''}, t'')]$, which is
intractable in closed form. Following \citet{daras2024consistent}, we draw two independent reverse-step samples $\tilde{\bx}_{t''}^1, \tilde{\bx}_{t''}^2 \sim p_\phi(\bx_{t''} \mid
\bx_{t'})$ to obtain the unbiased estimator:
\begin{equation}
\mathcal L_{\mathrm{cons}}^{(k)}(\phi)
\approx
\mathbb E_{\bx_{t'},\,\tilde{\bx}_{t''}^1,\tilde{\bx}_{t''}^2}
\!\left[
\left(
\hat{\bx}_{\phi}(\tilde{\bx}_{t''}^1,t'')
-
\hat{\bx}_{\phi}(\bx_{t'},t')
\right)\!^\top\!
\left(
\hat{\bx}_{\phi}(\tilde{\bx}_{t''}^2,t'')
-
\hat{\bx}_{\phi}(\bx_{t'},t')
\right)
\right].
\label{eq:consistency_loss_in_practice}
\end{equation}
In our experiments, we set
$t' = t_k$ and $t'' = t_{k+1}$ (adjacent endpoints of the timestep schedule), following \citet{daras2024consistent}.

\section{Hybrid Policy Gradient}
\label{app:hybrid_pg}

\subsection{Derivation}

We provide a formal derivation of the gradient estimator used for the reward term
\[
\nabla_\theta \E_{\hat{\bx}_0\sim p_\theta}[r(\hat{\bx}_0)].
\]
Since the few-step generator consists of \(K-1\) stochastic transitions followed by one deterministic final step, the resulting gradient decomposes naturally into a contribution from the stochastic transitions and a contribution from the deterministic terminal mapping.

\begin{proposition}[Hybrid policy gradient]
\label{prop:hybrid_pg}
Consider the few-step generative process
\[
\hat{\bx}_{t_0}\to \hat{\bx}_{t_1}\to \cdots \to \hat{\bx}_{t_{K-1}}\to \hat{\bx}_0,
\]
where for \(k=1,\dots,K-1\),
\[
\pi_\theta^{(k)}(\hat{\bx}_{t_k}\mid \hat{\bx}_{t_{k-1}})
=
\mathcal N\!\bigl(
\hat{\bx}_{t_k};
\bmu_\theta^{(k)}(\hat{\bx}_{t_{k-1}}),
\sigma_k^2\mathbf I
\bigr),
\]
and the final output is deterministic:
\[
\hat{\bx}_0 = G_\theta(\hat{\bx}_{t_{K-1}}, t_{K-1}).
\]
Define
\[
J(\theta):=\E_{\hat{\bx}_0\sim p_\theta}[r(\hat{\bx}_0)].
\]
Then
\begin{equation}
\begin{aligned}
\nabla_\theta J(\theta)
&=
\sum_{k=1}^{K-1}
\E_{\tau\sim p_\theta}\!\left[
r(\hat{\bx}_0)\,
\nabla_\theta \log \pi_\theta^{(k)}(\hat{\bx}_{t_k}\mid \hat{\bx}_{t_{k-1}})
\right] \\
&\quad+
\E_{\tau\sim p_\theta}\!\left[
\bigl(\nabla_{\hat{\bx}_0}r(\hat{\bx}_0)\bigr)^\top
\partial_\theta G_\theta(\hat{\bx}_{t_{K-1}}, t_{K-1})
\right].
\end{aligned}
\label{eq:app_hybrid_pg_prop}
\end{equation}
\end{proposition}

\begin{proof}
Because the last step is deterministic, it is convenient to separate the stochastic and deterministic parts of the generation process. Let
\[
\zeta := (\hat{\bx}_{t_1},\dots,\hat{\bx}_{t_{K-1}})
\]
denote the intermediate random states. Conditioned on \(\hat{\bx}_{t_0}\), their density is
\begin{equation}
p_\theta(\zeta)
=
\prod_{k=1}^{K-1}
\pi_\theta^{(k)}(\hat{\bx}_{t_k}\mid \hat{\bx}_{t_{k-1}}).
\label{eq:app_zeta_density_prop}
\end{equation}
The terminal sample is then determined by
\[
\hat{\bx}_0 = G_\theta(\hat{\bx}_{t_{K-1}}, t_{K-1}),
\]
and the objective can be written as
\begin{equation}
J(\theta)
=
\E_{\zeta\sim p_\theta}\!\left[
r\bigl(G_\theta(\hat{\bx}_{t_{K-1}}, t_{K-1})\bigr)
\right].
\label{eq:app_obj_prop}
\end{equation}

Differentiating Eq.~\eqref{eq:app_obj_prop} under the integral sign gives
\begin{equation}
\begin{aligned}
\nabla_\theta J(\theta)
&=
\nabla_\theta
\int
p_\theta(\zeta)\,
r\bigl(G_\theta(\hat{\bx}_{t_{K-1}}, t_{K-1})\bigr)\, d\zeta \\
&=
\int
\nabla_\theta p_\theta(\zeta)\, r(\hat{\bx}_0)\, d\zeta
+
\int
p_\theta(\zeta)\, \nabla_\theta r(\hat{\bx}_0)\, d\zeta.
\end{aligned}
\label{eq:app_split_prop_1}
\end{equation}
Using
\[
\nabla_\theta p_\theta(\zeta)
=
p_\theta(\zeta)\nabla_\theta \log p_\theta(\zeta),
\]
we obtain
\begin{equation}
\nabla_\theta J(\theta)
=
\E_{\zeta\sim p_\theta}\!\left[
r(\hat{\bx}_0)\nabla_\theta \log p_\theta(\zeta)
\right]
+
\E_{\zeta\sim p_\theta}\!\left[
\nabla_\theta r(\hat{\bx}_0)
\right].
\label{eq:app_split_prop_2}
\end{equation}

By Eq.~\eqref{eq:app_zeta_density_prop},
\[
\log p_\theta(\zeta)
=
\sum_{k=1}^{K-1}
\log \pi_\theta^{(k)}(\hat{\bx}_{t_k}\mid \hat{\bx}_{t_{k-1}}),
\]
and therefore
\begin{equation}
\nabla_\theta \log p_\theta(\zeta)
=
\sum_{k=1}^{K-1}
\nabla_\theta \log \pi_\theta^{(k)}(\hat{\bx}_{t_k}\mid \hat{\bx}_{t_{k-1}}).
\label{eq:app_score_sum_prop}
\end{equation}
Substituting Eq.~\eqref{eq:app_score_sum_prop} into Eq.~\eqref{eq:app_split_prop_2} yields
\begin{equation}
\nabla_\theta J(\theta)
=
\sum_{k=1}^{K-1}
\E_{\zeta\sim p_\theta}\!\left[
r(\hat{\bx}_0)\,
\nabla_\theta \log \pi_\theta^{(k)}(\hat{\bx}_{t_k}\mid \hat{\bx}_{t_{k-1}})
\right]
+
\E_{\zeta\sim p_\theta}\!\left[
\nabla_\theta r(\hat{\bx}_0)
\right].
\label{eq:app_split_prop_3}
\end{equation}

For the second term, \(r(\hat{\bx}_0)\) depends on \(\theta\) through the final deterministic mapping \(G_\theta\). Hence, by the chain rule,
\begin{equation}
\nabla_\theta r(\hat{\bx}_0)
=
\bigl(\nabla_{\hat{\bx}_0}r(\hat{\bx}_0)\bigr)^\top
\partial_\theta G_\theta(\hat{\bx}_{t_{K-1}}, t_{K-1}).
\label{eq:app_chain_prop}
\end{equation}
Combining Eqs.~\eqref{eq:app_split_prop_3} and~\eqref{eq:app_chain_prop} proves Eq.~\eqref{eq:app_hybrid_pg_prop}.
\end{proof}

\subsection{Discussion}

Prop.~\ref{prop:hybrid_pg} separates the gradient into two parts. The first part comes from how the parameters affect the distribution of the stochastic intermediate states, and therefore involves the log-derivatives of the transition densities. The second part comes from the explicit dependence of the final deterministic output on the parameters, and is obtained by directly differentiating the terminal mapping.

A subtle point is that the above derivation applies the log-derivative identity only to the stochastic part of the trajectory. One could formally write the full trajectory distribution as
\[
p_\theta(\tau)
=
p(\hat{\bx}_{t_0})
\prod_{k=1}^{K-1}
\pi_\theta^{(k)}(\hat{\bx}_{t_k}\mid \hat{\bx}_{t_{k-1}})
\cdot
\delta\!\left(
\hat{\bx}_0-G_\theta(\hat{\bx}_{t_{K-1}}, t_{K-1})
\right),
\]
but this representation contains a Dirac delta corresponding to the deterministic last step. For this reason, directly applying the log-derivative form to the entire trajectory is unnecessary and less clean. By isolating the random intermediate states \(\zeta\), the derivation remains fully standard and avoids manipulating the logarithm of a Dirac delta.

Another potential concern is whether the second term in Eq.~\eqref{eq:app_hybrid_pg_prop} ignores the dependence of \(\hat{\bx}_{t_{K-1}}\) on \(\theta\). It does not. In Eq.~\eqref{eq:app_split_prop_2}, the dependence of the random trajectory \(\zeta\) on \(\theta\) is already accounted for by the term
$\E_{\zeta\sim p_\theta}\!\left[
r(\hat{\bx}_0)\nabla_\theta \log p_\theta(\zeta)
\right]$.
Therefore, in the remaining term
$\E_{\zeta\sim p_\theta}\!\left[
\nabla_\theta r(\hat{\bx}_0)
\right]$,
the sampled \(\zeta\) is treated as fixed, and only the explicit parameter dependence of the terminal mapping \(G_\theta\) is differentiated. Re-introducing the dependence of \(\hat{\bx}_{t_{K-1}}\) on \(\theta\) in this term would count the same contribution twice.

\section{Background on GRPO and Its Application to Few-step Generators}\label{app:grpo_background}                                                                                     
                                 
\noindent\textbf{Group relative policy optimization.}                                                                                                                                  
GRPO~\citep{deepseekmath} is a variance-reduced policy gradient method originally developed for language model alignment. Given a prompt $c$, GRPO samples a group of $N$ responses,   
computes a group-normalized advantage for each, and updates the policy with a clipped objective. The key idea is to use the group mean reward as a baseline, eliminating the need for a
separate value network.                                                                                                                                                               
                                                        
\noindent\textbf{Naive GRPO for few-step generators.}
In our setting, the few-step generator defines a $K$-step policy $\pi_\theta$ with Gaussian transitions (Sec.~\ref{sec:reinforce_gen}). A direct (naive) application of GRPO proceeds
as follows: for each prompt $c$, sample $N$ independent trajectories $\{\tau_i\}_{i=1}^N$ by drawing independent noise $\bepsilon_k^{(i)} \sim \mathcal{N}(\mathbf{0}, \mathbf{I})$ at 
every stochastic step $k = 1,\dots,K-1$ of each trajectory $i$. Compute rewards $r_i = r(\hat{\bx}_0^{(i)}, c)$ and group-normalized advantages
\[                                                                                                                                                                                     
A_i = \frac{r_i - \bar{r}}{\operatorname{std}(\{r_j\}_{j=1}^N)}, \qquad \bar{r} = \frac{1}{N}\sum_{j=1}^N r_j.
\]
The naive GRPO gradient estimator for the stochastic part is then
\begin{equation}
\nabla_\theta \mathcal{L}_{\mathrm{stoc}}^{\mathrm{GRPO}}
=
\frac{1}{N}\sum_{i=1}^N
\sum_{k=1}^{K-1}
A_i\,
\nabla_\theta \log \pi_\theta^{(k)}(\hat{\bx}_{t_k}^{(i)} \mid \hat{\bx}_{t_{k-1}}^{(i)}),
\label{eq:naive_grpo}
\end{equation}
where the Gaussian log-likelihood is  \[
\log \pi_\theta^{(k)}(\hat{\bx}_{t_k} \mid \hat{\bx}_{t_{k-1}})                                                                                                                        
= -\frac{1}{2\sigma_k^2}\|\hat{\bx}_{t_k} - \bmu_\theta^{(k)}(\hat{\bx}_{t_{k-1}})\|^2 + \mathrm{const},
\]
with $\bmu_\theta^{(k)}$ and $\sigma_k = t_k\sin(\eta\pi/2)$ determined by the CPS update (App.~\ref{app:cps_formula}).

\noindent\textbf{Limitation.}
Since all $K-1$ steps use independent noise across trajectories, the reward difference between any two trajectories reflects the combined effect of noise injected at \emph{all} steps
simultaneously. The advantage $A_i$ is therefore a noisy credit assignment signal, which means it cannot attribute reward variation to any specific step, leading to high gradient variance. This
motivates our SubGRPO (Sec.~\ref{sec:reinforce_gen}), which shares noise at non-selected steps to isolate contributions from a chosen subset.

\section{Detailed Algorithm of RTDMD}\label{app:algo}
We present a detailed algorithm pipeline of RTDMD in Algo.~\ref{alg:rtdmd}.
\begin{algorithm}[!ht]
\caption{RTDMD: Reward-Tilted Distribution Matching Distillation}
\label{alg:rtdmd}
\begin{algorithmic}[1]
\Require Teacher score $s_{\psi}$, generator $G_{\theta}$, fake score $s_{\phi}$, reward $r(\cdot)$, $K$-step schedule $\{t_k\}_{k=0}^K$, reward weight $\beta$.
\Ensure Reward-optimized $K$-step generator $G_{\theta}$.

\vspace{0.4em}
\Statex \textbf{Stage I: Ambient-Consistent Distribution Matching Distillation (cold start)}
\vspace{0.2em}
\For{$u=1,\ldots,T_{\mathrm{cold}}$}
    \State Roll out $G_{\theta}$ from noise; sample step index $k$ and construct noised state $x^{(k)}_t$ for $t\in[t_k,1]$.
    \State Update $s_{\phi}$ via the fake score objective (Eq.~\eqref{eq:fake_score_total}).
    \State Update $G_{\theta}$ via the AC-DMD generator objective (Eq.~\eqref{eq:subinterval-generator-gradient}).
\EndFor

\vspace{0.4em}
\Statex \textbf{Stage II: Reinforcing the Few-step Generator}
\vspace{0.2em}
\For{$u=1,\ldots,T_{\mathrm{rl}}$}
    \For{each prompt $c_j$ in batch}
        \State Sample step subset $S\!\subset\!\{1,\ldots,K\!-\!1\}$.
        \State Generate $N$ trajectories with shared/independent noise (Eq.~\eqref{eq:shared-noise}).
        \State Compute rewards $r_i = r(\hat{x}^{(i)}_0, c_j)$ and group-normalized advantages $A_i$.
    \EndFor
    \State Update $s_{\phi}$ via the fake score objective (Eq.~\eqref{eq:fake_score_total}).
    \State Update $G_{\theta}$ by descending along $\nabla_\theta\mathcal{L}_{\mathrm{total}}$ (Eq.~\eqref{eq:total_loss}).
\EndFor

\State \Return $G_{\theta}$.
\end{algorithmic}
\end{algorithm}

\section{More Implementation Details}\label{app:more_details}
For the 4-step generator ($K\!=\!4$), we uniformly partition the interval $[0,1]$ to obtain the pre-shift timestep schedule $[1.0, 0.75, 0.5, 0.25]$. For FLUX.2 4B, the Stage~1 generator is initialized from the FLUX.2 [klein] 4B checkpoint rather than the Base 4B model. All experiments use AdamW~\citep{loshchilov2018decoupled} with $\beta_1\!=\!0.9$,
$\beta_2\!=\!0.999$, and a constant learning rate. In Stage~1 (cold start), we use the same learning rate for the fake score model and generator. The learning rate is $1\!\times\!10^{-6}$ for FLUX.2 4B and $3\!\times\!10^{-5}$ for the SD series; in Stage~2, the fake score model retains its Stage~1 learning rate
while the generator learning rate is uniformly set to $3\!\times\!10^{-6}$. LoRA~\citep{lora} is applied to all linear layers within self-attention and cross-attention modules ($r\!=\!64$, $\alpha\!=\!32$). The GRPO clipping range is $[1\!-\!10^{-5},\;1\!+\!10^{-5}]$, and the reward-tilting coefficient $\beta$ is set to $1.0$ throughout. For multi-reward training, each reward is weighted equally. All training and rollout use \texttt{bf16} precision. Stage~1 uses 8 NVIDIA H20 GPUs for all models; Stage~2 uses 16 GPUs for FLUX.2 4B and 8 GPUs for the SD series.

\section{Discussion on the Sampling Schedule and Connection to DMD2}\label{app:discussion_sampling}

\noindent\textbf{Choice of $\eta$.}
As shown in Fig.~\ref{fig:eta}, we find that performance is relatively insensitive to the specific value of $\eta$ as long as sufficient stochasticity is introduced (\emph{i.e.}, $\eta\!\geq\!0.8$), with all such
configurations yielding comparable results. We attribute this to the fact that the subsequent GRPO optimization relies on adequate stochasticity in the intermediate transitions to produce diverse trajectories for effective exploration. We therefore set $\eta\!=\!0.9$ throughout our experiments.

\noindent\textbf{Connection and comparison to DMD2.}
We further find that the only difference between the original DMD2~\citep{dmd2} (which originally uses the CM scheduler) and our A-DMD with $\eta\!=\!1.0$ (where CPS reduces to the CM scheduler) lies in the noise-level interval used during training: $t$ in Eq.~\eqref{eq:gen_kl} ranges over $[0,1]$ for DMD2 but only over $[t_k,1]$ for A-DMD. This stems from a different choice of conditioning data for score/denoising matching: A-DMD re-noises from the actually generated intermediate state $\hat{\bx}_{t_k}$ and uses it directly as the training data, whereas DMD2 instead takes the generator's $x$-prediction $\hat{\bx}^{(k)}_\mathrm{pred}$ as a clean sample $\bx_0$ and re-noises it across the entire interval $[0,1]$.

Our design is more general across different schedulers (\emph{e.g.}, CM and Euler ODE), because $\hat{\bx}^{(k)}_\mathrm{pred}$ can be safely treated as a clean sample only under the CM parameterization. In a CM-trained model, the consistency property forces the network to map any noisy input back to the same underlying clean $\bx_0$ along the PF-ODE trajectory, so $\hat{\bx}^{(k)}_\mathrm{pred}$ behaves as a sample from the student's clean-data distribution. Under non-CM schedulers (\emph{e.g.}, the Euler ODE schedule with $\eta\!=\!0$), the model does not satisfy this consistency property and is not trained to directly map noisy inputs to clean samples, so $\hat{\bx}^{(k)}_\mathrm{pred}$ cannot be regarded as a sample from $p_0$. Treating it as if it were a clean $\bx_0$ and re-noising it over $[0,1]$ would therefore introduce a scheduler-dependent bias in the score-matching objective~\citep{fan2026phaseddmdfewstepdistribution}. By instead conditioning on the realized noisy state $\hat{\bx}_{t_k}$, which is by construction a sample from the student marginal $p_\theta^{(k)}$ (defined in Sec.~\ref{sec:ac-dm}), and matching the score only on the subinterval $[t_k,1]$, A-DMD remains a valid re-derivation of distribution matching for any $\eta\in[0,1]$ and any scheduler choice.

Empirically, in the special case $\eta\!=\!1.0$ A-DMD becomes structurally equivalent to DMD2 up to this training-interval choice, and as shown in Fig.~\ref{fig:eta} the corresponding curve converges to nearly identical final performance as the standard DMD2 setting. This confirms that A-DMD does not sacrifice quality in the CM regime; rather, the re-derivation is introduced primarily to provide a unified, scheduler-agnostic framework.
\begin{figure}[!ht]
   \centering
   \setlength{\abovecaptionskip}{0.2cm}
   \begin{minipage}[b]{0.5\linewidth}
        \centering
        \begin{subfigure}[tp!]{\linewidth}
        \centering
        \includegraphics[width=\linewidth]{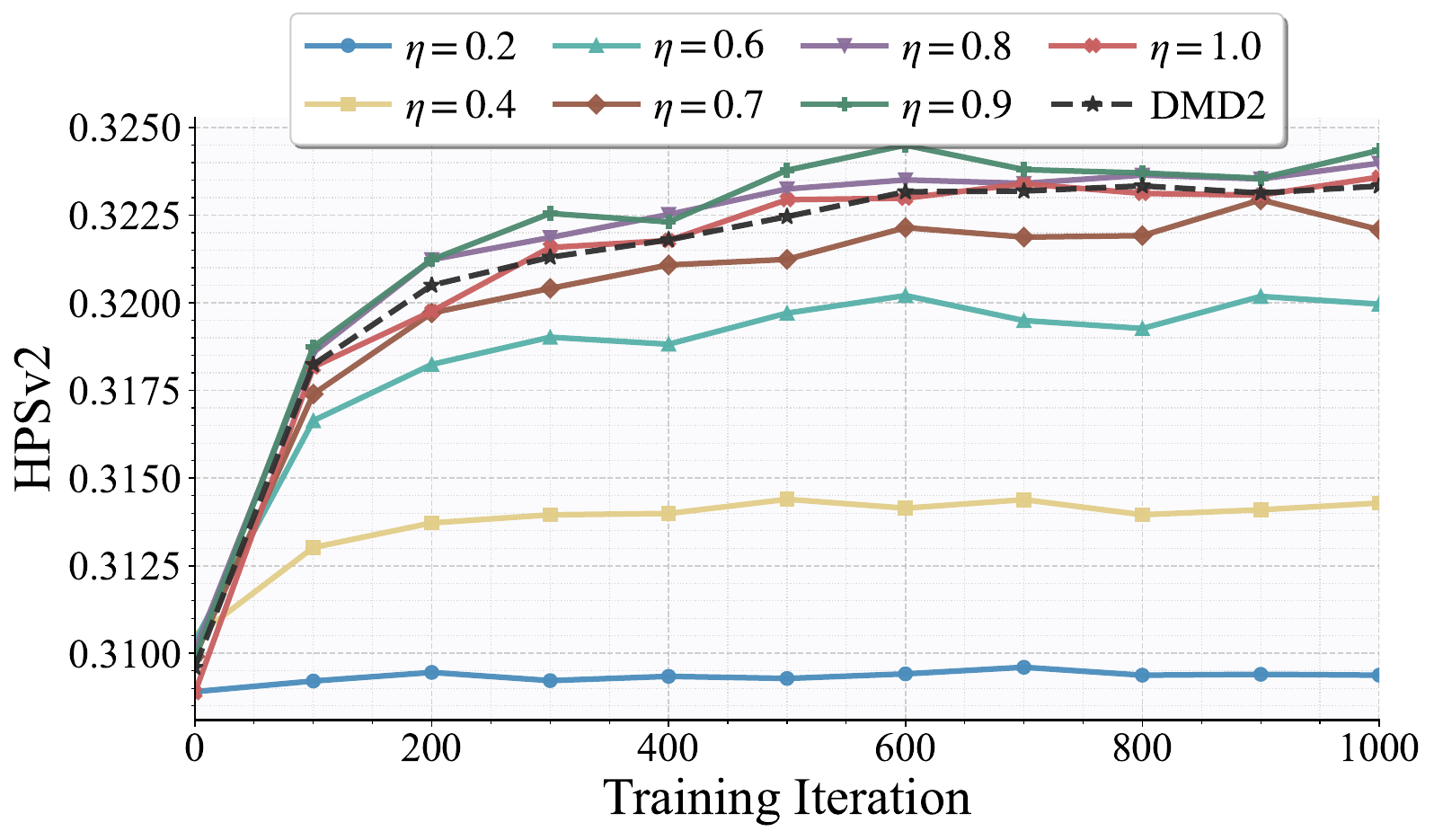}
        \end{subfigure}
   \end{minipage}
   \caption{\label{fig:eta} Evaluation curves when reinforcing the few-step generators with different $\eta$. We use \texttt{SD3.5-M}~\citep{sd35} with HPSv2~\citep{hpsv2} as the sole training reward. Each curve is optimized by A-DMD with different $\eta$ values combined with GRPO on stochastic transitions.}
\end{figure}

\begin{table}[!ht]\setlength{\tabcolsep}{2pt}
 \renewcommand{\arraystretch}{1.1}
  \centering
  \caption{GenEval results. Results for models other than \texttt{SD3.5-M} are from \citet{luo2026tdmr1reinforcingfewstepdiffusion} or original papers.} 
  \resizebox{0.70\linewidth}{!}{
  \begin{tabu}{l|c|cccccc}
\toprule
\makecell{\textbf{Model}} & \makecell{\textbf{Overall}} & \makecell{\textbf{Single}\\\textbf{Object}} & \makecell{\textbf{Two}\\\textbf{Object}} & \makecell{\textbf{Counting}} & \makecell{\textbf{Colors}} &
  \makecell{\textbf{Position}} & \makecell{\textbf{Attribution}\\\textbf{Binding}} \\
\midrule
\multicolumn{8}{l}{\cellcolor[gray]{0.92}\textbf{\textit{Autoregressive Models}}} \\
\midrule
Show-o~\citep{xie2025showosingletransformerunify}  & 0.53 & 0.95 & 0.52 & 0.49 & 0.82 & 0.11 & 0.28 \\
Emu3-Gen~\citep{wang2024emu3nexttokenpredictionneed}  & 0.54 & 0.98 & 0.71 & 0.34 & 0.81 & 0.17 & 0.21 \\
JanusFlow~\citep{ma2025janusflowharmonizingautoregressionrectified}  & 0.63 & 0.97 & 0.59 & 0.45 & 0.83 & 0.53 & 0.42 \\
Janus-Pro-7B~\citep{chen2025janusprounifiedmultimodalunderstanding}  & 0.80 & 0.99 & 0.89 & 0.59 & 0.90 & 0.79 & 0.66 \\
GPT-4o~\citep{openai2024gpt4ocard} & 0.84 & 0.99 & 0.92 & 0.85 & 0.92 & 0.75 & 0.61 \\
\midrule
\multicolumn{8}{l}{\cellcolor[gray]{0.92}\textbf{\textit{Diffusion Models}}} \\
\midrule
LDM~\citep{ldm}  & 0.37 & 0.92 & 0.29 & 0.23 & 0.70 & 0.02 & 0.05 \\
\texttt{SD1.5}~\citep{ldm}  & 0.43 & 0.97 & 0.38 & 0.35 & 0.76 & 0.04 & 0.06 \\
\texttt{SD2.1}~\citep{ldm}  & 0.50 & 0.98 & 0.51 & 0.44 & 0.85 & 0.07 & 0.17 \\
\texttt{SD-XL}~\citep{sdxl}  & 0.55 & 0.98 & 0.74 & 0.39 & 0.85 & 0.15 & 0.23 \\
DALLE-2~\citep{dalle2}  & 0.52 & 0.94 & 0.66 & 0.49 & 0.77 & 0.10 & 0.19 \\
DALLE-3~\citep{betker2023improving}  & 0.67 & 0.96 & 0.87 & 0.47 & 0.83 & 0.43 & 0.45 \\
\texttt{FLUX.1 Dev}~\citep{flux}  & 0.66 & 0.98 & 0.81 & 0.74 & 0.79 & 0.22 & 0.45 \\
\texttt{SD3.5-L}~\citep{sd3}  & 0.71 & 0.98 & 0.89 & 0.73 & 0.83 & 0.34 & 0.47 \\
SANA-1.5 4.8B~\citep{xie2025sana} & 0.81 & 0.99 & 0.93 & 0.86 & 0.84 & 0.59 & 0.65 \\
\midrule
\texttt{SD3.5-M}~\citep{sd3} & 0.63 & 0.98 & 0.78 & 0.50 & 0.81 & 0.24 & 0.52 \\
w/ Flow-GRPO~\citep{flowgrpo} & 0.95 & 1.00 & 0.99 & 0.95 & 0.92 & 0.99 & 0.86  \\
\midrule
\multicolumn{8}{l}{\cellcolor[gray]{0.92}\textbf{\textit{Few-Step Diffusion Models (4 NFE)}}} \\
\midrule
\texttt{SD3.5-L-Turbo}~\citep{sd3} & 0.70 & 0.94 & 0.84 & 0.55 & 0.79 & 0.58 & 0.56 \\
\texttt{SD3.5-M} w/ TDM~\citep{tdm} & 0.61 & 0.99 & 0.77 & 0.49 & 0.79 & 0.23 & 0.44 \\
\texttt{SD3.5-M}  w/ TDM-R1~\citep{luo2026tdmr1reinforcingfewstepdiffusion} & 0.92 & 1.00 & 0.96 & 0.88 & 0.85 & 0.93 & 0.91\\
\rowcolor{colSky!30}\texttt{SD3.5-M}  w/ RTDMD (Ours) & \textbf{0.94} & \textbf{1.00} & \textbf{0.98} & \textbf{0.95} & \textbf{0.90} & \textbf{0.87} & \textbf{0.94}\\
\bottomrule
\end{tabu}
}
    \label{tab:sd35}
\end{table}
\section{GenEval Results}\label{app:geneval_results}
We further validate our framework using GenEval~\citep{geneval}, a non-differentiable compositional generation benchmark, on \texttt{SD3.5-M}. Since the reward signal is non-differentiable, the hybrid policy gradient cannot backpropagate through the final deterministic step; optimization relies solely on AC-DMD and SubGRPO for the stochastic transitions. Despite this, as shown in Tab.~\ref{tab:sd35}, our method achieves an overall score of 0.94 with only 4 NFE, nearly matching
Flow-GRPO~\citep{flowgrpo} (0.95), which operates on the full multi-step model with 40 NFE and CFG~\citep{ho2022classifierfreediffusionguidance}, while outperforming TDM-R1~\citep{luo2026tdmr1reinforcingfewstepdiffusion} (0.92) under the same 4-step setting.

\section{Qualitative comparison for Ablation}\label{app:vis_ablation}
We provide qualitative visualizations corresponding to each ablation study in the main text. Specifically, Figs.~\ref{fig:sd35_cold}--\ref{fig:sd35_grpo} present samples generated by the configurations listed in Tabs.~\ref{tab:ac-dmd}--\ref{tab:hybrid-policy-gradient}, respectively. None of the samples shown are cherry-picked, and all are produced using randomly drawn prompts and noise seeds. These visualizations offer direct visual evidence for the effectiveness of each proposed component, complementing the quantitative analysis in the main text.

\begin{figure}[!ht]
   \centering
   \setlength{\abovecaptionskip}{0.2cm}
   \begin{minipage}[b]{0.75\linewidth}
   \begin{spacing}{0.7}
           {\tiny Text prompt: \textit{`` A portrait of a young Mark Hamill as Luke Skywalker from "Star Wars: Return of the Jedi" in shades of grey with touches of green by Jeremy Mann.''}}
        \vspace{0.05in}         \end{spacing}
       \begin{minipage}[b]{\linewidth}
            \centering
            \begin{subfigure}[tp!]{\textwidth}
            \centering
            \includegraphics[width=\linewidth]{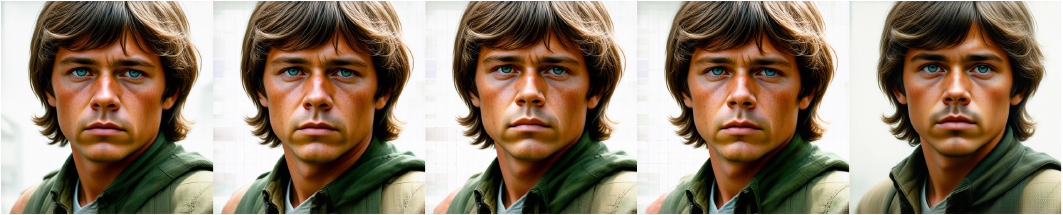}
            \end{subfigure}
       \end{minipage}
       \begin{spacing}{0.7}
       {\tiny Text prompt: \textit{``People are walking on the street on a rainy day.''}}
    \vspace{0.05in}         \end{spacing}
       \begin{minipage}[b]{\linewidth}
            \centering
            \begin{subfigure}[tp!]{\linewidth}
            \centering
            \includegraphics[width=\linewidth]{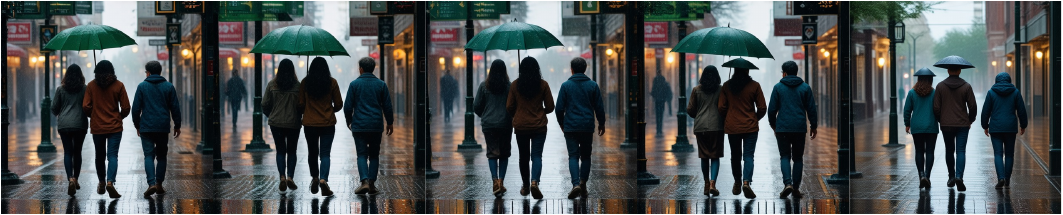}
            \end{subfigure}
       \end{minipage}
   \end{minipage}
    \vspace{-0.05in}
   \caption{Qualitative comparison of different distillation methods upon completion of cold-start training. From \textit{left} to \textit{right}: AC-DMD ($\gamma=0.001$), AC-DMD ($\gamma=0.005$), AC-DMD ($\gamma=0.01$), AC-DMD ($\gamma=0.1$), A-DMD. The columns correspond to the configurations in Tab.~\ref{tab:ac-dmd}, listed in bottom-to-top order.}
   \label{fig:sd35_cold}
\end{figure}
\begin{figure}[!ht]
   \centering
   \setlength{\abovecaptionskip}{0.2cm}
   \begin{minipage}[b]{0.75\linewidth}
   \begin{spacing}{0.7}
           {\tiny Text prompt: \textit{``The image is a close-up portrait of a demon goddess with tribal elements and intricate artwork by multiple artists.''}}
        \vspace{0.05in}         \end{spacing}
       \begin{minipage}[b]{\linewidth}
            \centering
            \begin{subfigure}[tp!]{\textwidth}
            \centering
            \includegraphics[width=\linewidth]{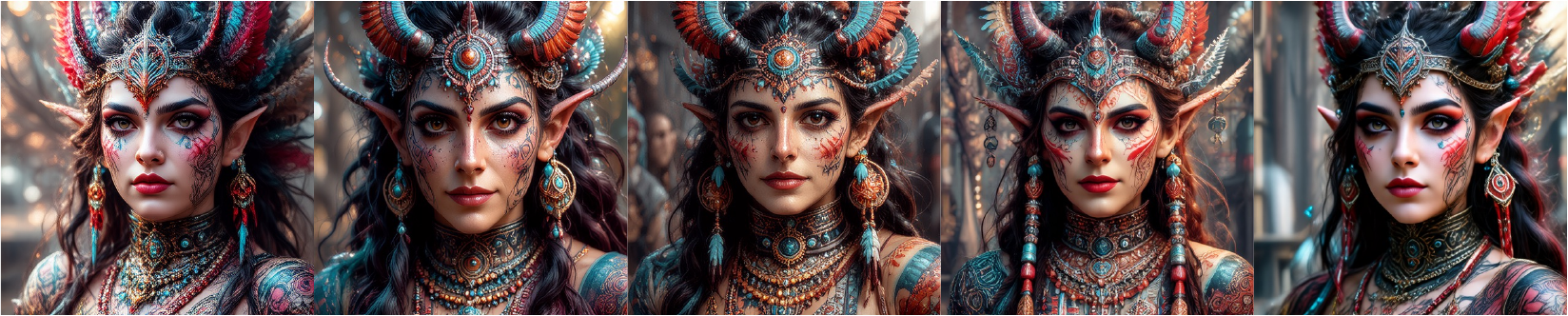}
            \end{subfigure}
       \end{minipage}
       \begin{spacing}{0.7}
       {\tiny Text prompt: \textit{``A triangular pink stop sign. A pink stop sign in the shape of a triangle.''}}
    \vspace{0.05in}         \end{spacing}
       \begin{minipage}[b]{\linewidth}
            \centering
            \begin{subfigure}[tp!]{\linewidth}
            \centering
            \includegraphics[width=\linewidth]{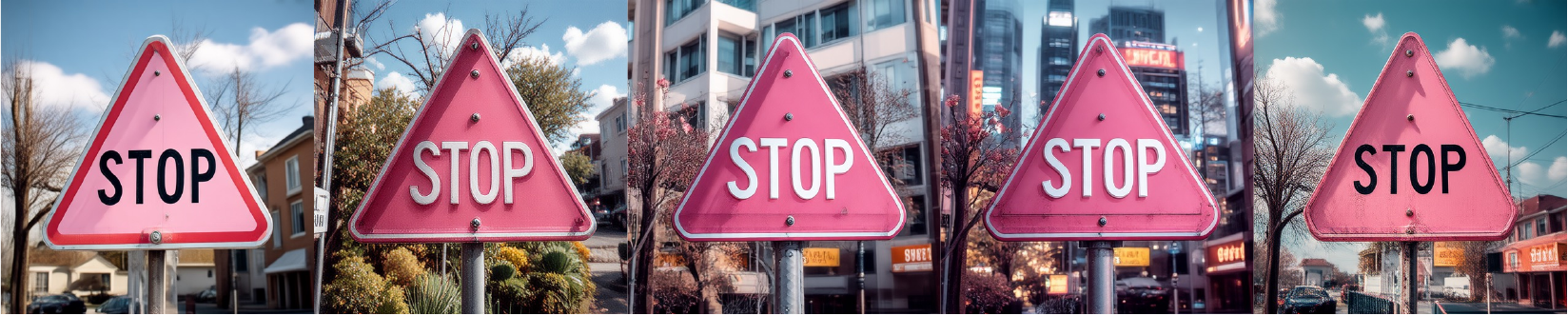}
            \end{subfigure}
       \end{minipage}
   \end{minipage}
    \vspace{-0.05in}
   \caption{Qualitative comparison of different distillation methods upon completion of two-stage training. From \textit{left} to \textit{right}: AC-DMD ($\gamma=0.001$), AC-DMD ($\gamma=0.005$), AC-DMD ($\gamma=0.01$), AC-DMD ($\gamma=0.1$), A-DMD. The columns correspond to the rows in Tab.~\ref{tab:ac-dmd} in bottom-to-top order.}
   \label{fig:sd35_dmd}
\end{figure}
\begin{figure}[!ht]
   \centering
   \setlength{\abovecaptionskip}{0.2cm}
   \begin{minipage}[b]{0.75\linewidth}
   \begin{spacing}{0.7}
           {\tiny Text prompt: \textit{``A fluffy baby sloth with a knitted hat trying to figure out a laptop, close up, highly detailed, studio lighting, screen reflecting in its eyes.''}}
        \vspace{0.05in}         \end{spacing}
       \begin{minipage}[b]{\linewidth}
            \centering
            \begin{subfigure}[tp!]{\textwidth}
            \centering
            \includegraphics[width=\linewidth]{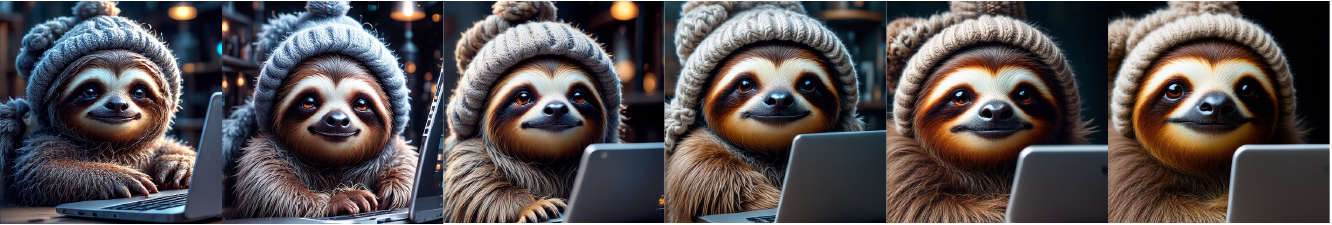}
            \end{subfigure}
       \end{minipage}
       \begin{spacing}{0.7}
       {\tiny Text prompt: \textit{``Octothorpe.''}}
    \vspace{0.05in}         \end{spacing}
       \begin{minipage}[b]{\linewidth}
            \centering
            \begin{subfigure}[tp!]{\linewidth}
            \centering
            \includegraphics[width=\linewidth]{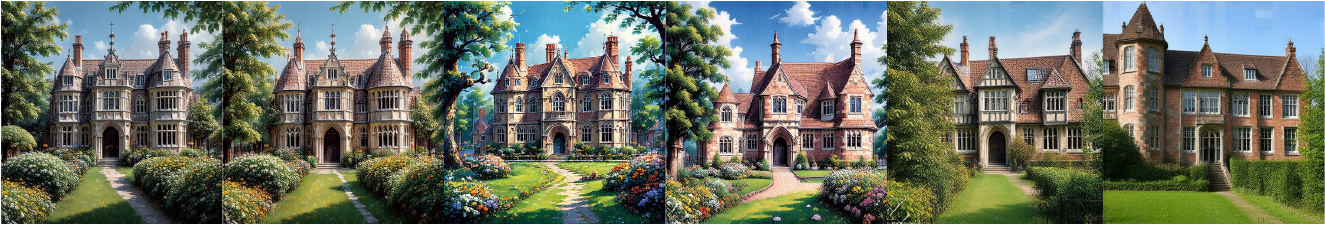}
            \end{subfigure}
       \end{minipage}
   \end{minipage}
    \vspace{-0.05in}
   \caption{Qualitative comparison of different reinforcement learning methods upon completion of two-stage training. From \textit{left} to \textit{right}: RTDMD ($M=2$), RTDMD ($M=2$) w/o $\mathcal{L}_{\mathrm{det}}$, RTDMD ($M=1$), RTDMD w/o $\mathcal{L}_{\mathrm{det}}$, GRPO~\citep{deepseekmath}, and $\emptyset$. The columns correspond to the rows in Tab.~\ref{tab:hybrid-policy-gradient} in bottom-to-top order.}
   \label{fig:sd35_grpo}
\end{figure}

\section{More Qualitative Results}\label{app:more_qualitative}
We provide additional visual comparisons in Fig.~\ref{fig:sd3_compare}, Fig.~\ref{fig:flux_compare}, and Fig.~\ref{fig:supple_flux}. Across diverse prompts, our RTDMD consistently produces images with superior visual quality and prompt adherence compared to baselines.
\begin{figure}[!ht]
   \centering
   \setlength{\abovecaptionskip}{0.2cm}
   \begin{minipage}[b]{0.75\linewidth}
   \begin{spacing}{0.7}
           {\tiny Text prompt: \textit{`` portrait of a beautiful woman wearing a futuristic headdress with daisies, puffballs, mushrooms, and other organic shapes.''}}
        \vspace{0.05in}         \end{spacing}
       \begin{minipage}[b]{\linewidth}
            \centering
            \begin{subfigure}[tp!]{\textwidth}
            \centering
            \includegraphics[width=\linewidth]{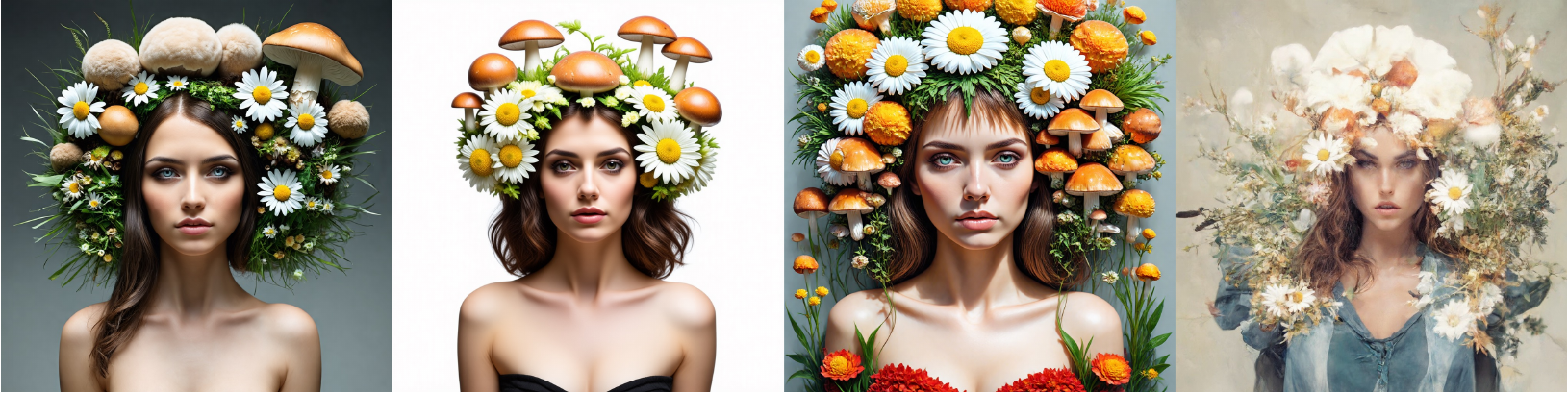}
            \end{subfigure}
       \end{minipage}
       \begin{spacing}{0.7}
       {\tiny Text prompt: \textit{``Lee Jin-eun in cyberpunk-themed photograph emerging from pink water.''}}
    \vspace{0.05in}         \end{spacing}
       \begin{minipage}[b]{\linewidth}
            \centering
            \begin{subfigure}[tp!]{\linewidth}
            \centering
            \includegraphics[width=\linewidth]{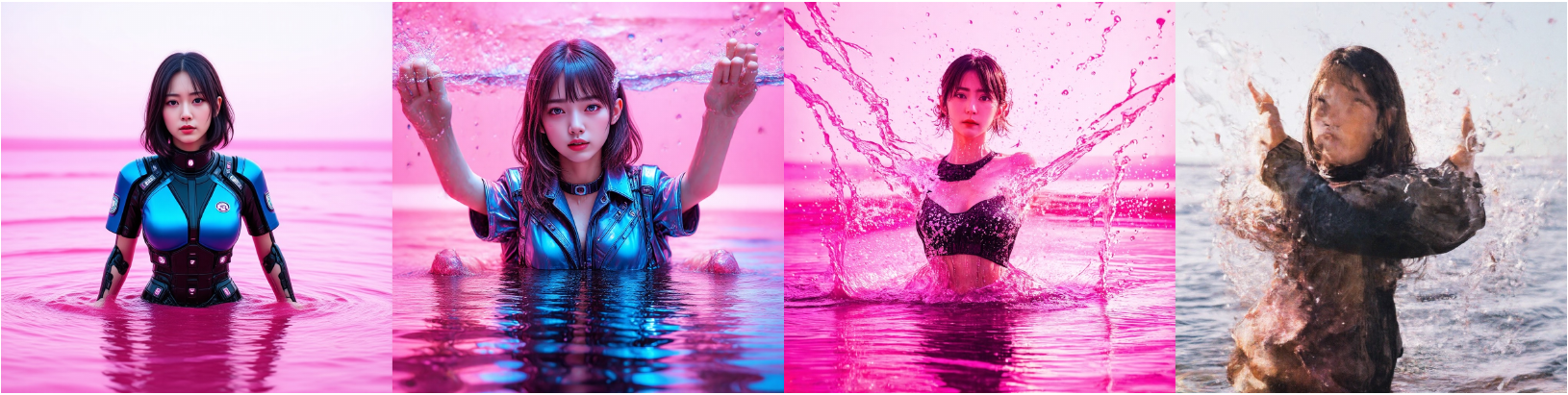}
            \end{subfigure}
       \end{minipage}
       \begin{spacing}{0.7}
       {\tiny Text prompt: \textit{``a couple of horse that are eating some grass.''}}
    \vspace{0.05in}         \end{spacing}
       \begin{minipage}[b]{\linewidth}
            \centering
            \begin{subfigure}[tp!]{\linewidth}
            \centering
            \includegraphics[width=\linewidth]{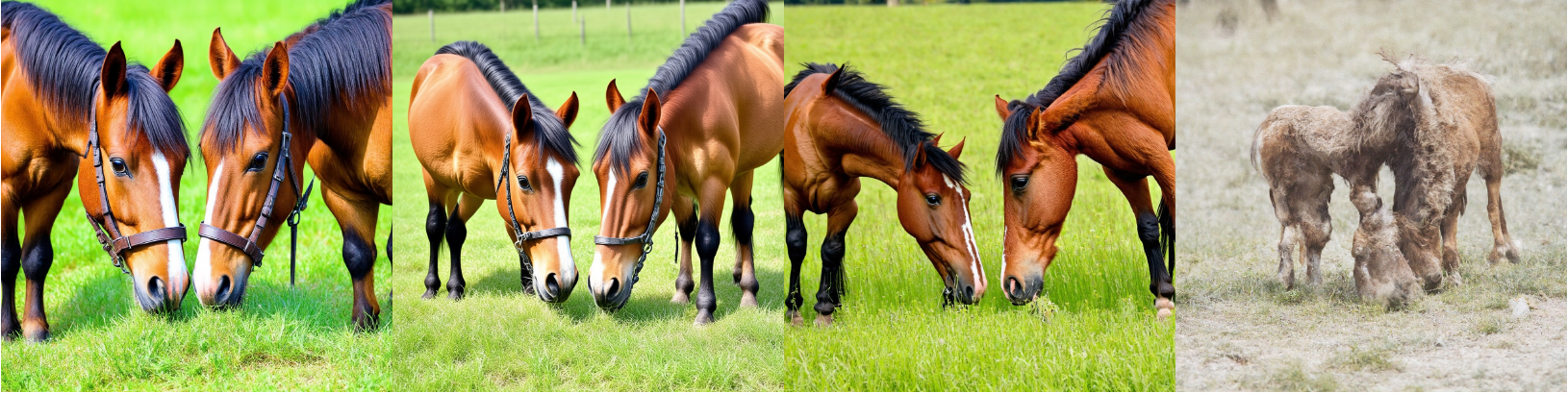}
            \end{subfigure}
       \end{minipage}
       \begin{spacing}{0.7}
       {\tiny Text prompt: \textit{``Small personal bathroom with a tiny entrance door.''}}
    \vspace{0.05in}         \end{spacing}
       \begin{minipage}[b]{\linewidth}
            \centering
            \begin{subfigure}[tp!]{\linewidth}
            \centering
            \includegraphics[width=\linewidth]{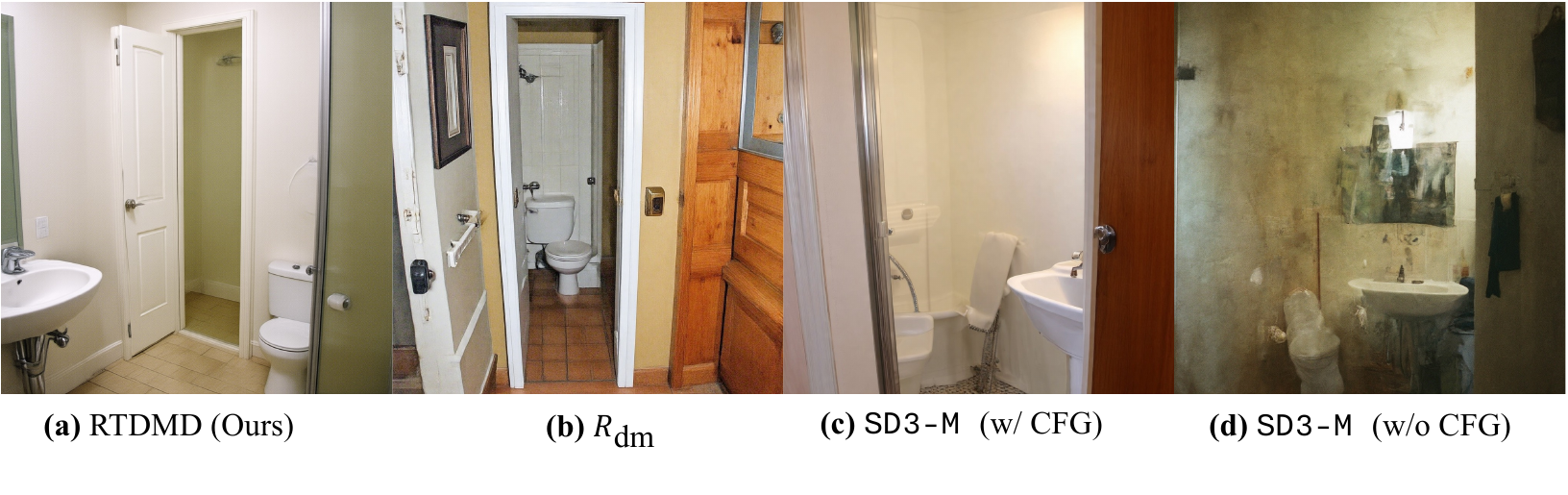}
            \end{subfigure}
       \end{minipage}
   \end{minipage}
    \vspace{-0.15in}
   \caption{Qualitative comparison for \texttt{SD3-M}~\citep{esser2024scaling}. Using identical noise inputs, our method outperforms others in both quality and prompt alignment, showing strong performance.}
   \label{fig:sd3_compare}
\end{figure}

\begin{figure}[!ht]
   \centering
   \setlength{\abovecaptionskip}{0.2cm}
   \begin{minipage}[b]{0.8\linewidth}
   \begin{spacing}{0.7}
       {\tiny Text prompt: \textit{``A maglev train going vertically downward in high speed, New York Times photojournalism.''}}
    \vspace{0.05in}         \end{spacing}
       \begin{minipage}[b]{\linewidth}
            \centering
            \begin{subfigure}[tp!]{\textwidth}
            \centering
            \includegraphics[width=\linewidth]{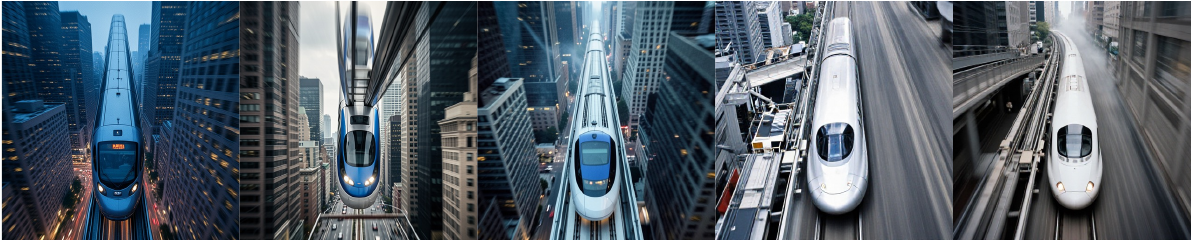}
            \end{subfigure}
       \end{minipage}\hfill
       \begin{spacing}{0.7}
       {\tiny Text prompt: \textit{``A white squirrel on a rocket in space.''}}
    \vspace{0.05in}         \end{spacing}
       \begin{minipage}[b]{\linewidth}
            \centering
            \begin{subfigure}[tp!]{\linewidth}
            \centering
            \includegraphics[width=\linewidth]{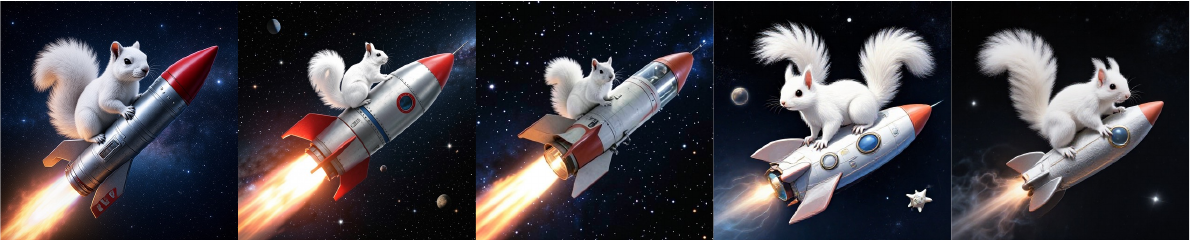}
            \end{subfigure}
       \end{minipage}
       \begin{spacing}{0.7}
       {\tiny Text prompt: \textit{``A 3D portrait of anime schoolgirls with grey hair submerged in dark water with dramatic lighting.''}}
    \vspace{0.05in}         \end{spacing}
       \begin{minipage}[b]{\linewidth}
            \centering
            \begin{subfigure}[tp!]{\linewidth}
            \centering
            \includegraphics[width=\linewidth]{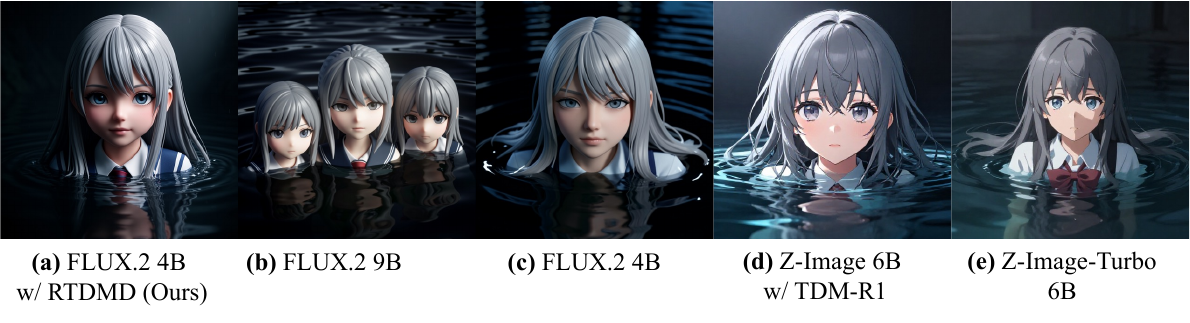}
            \end{subfigure}
       \end{minipage}
   \end{minipage}
    \vspace{-0.1in}
   \caption{Qualitative comparison for few-step diffusion models (4 NFE). Using identical noise inputs, our method outperforms others in both quality and prompt alignment, showing strong performance.}
   \label{fig:flux_compare}
\end{figure}

\begin{figure}[!ht]
   \centering
     \includegraphics[width=0.75\textwidth]{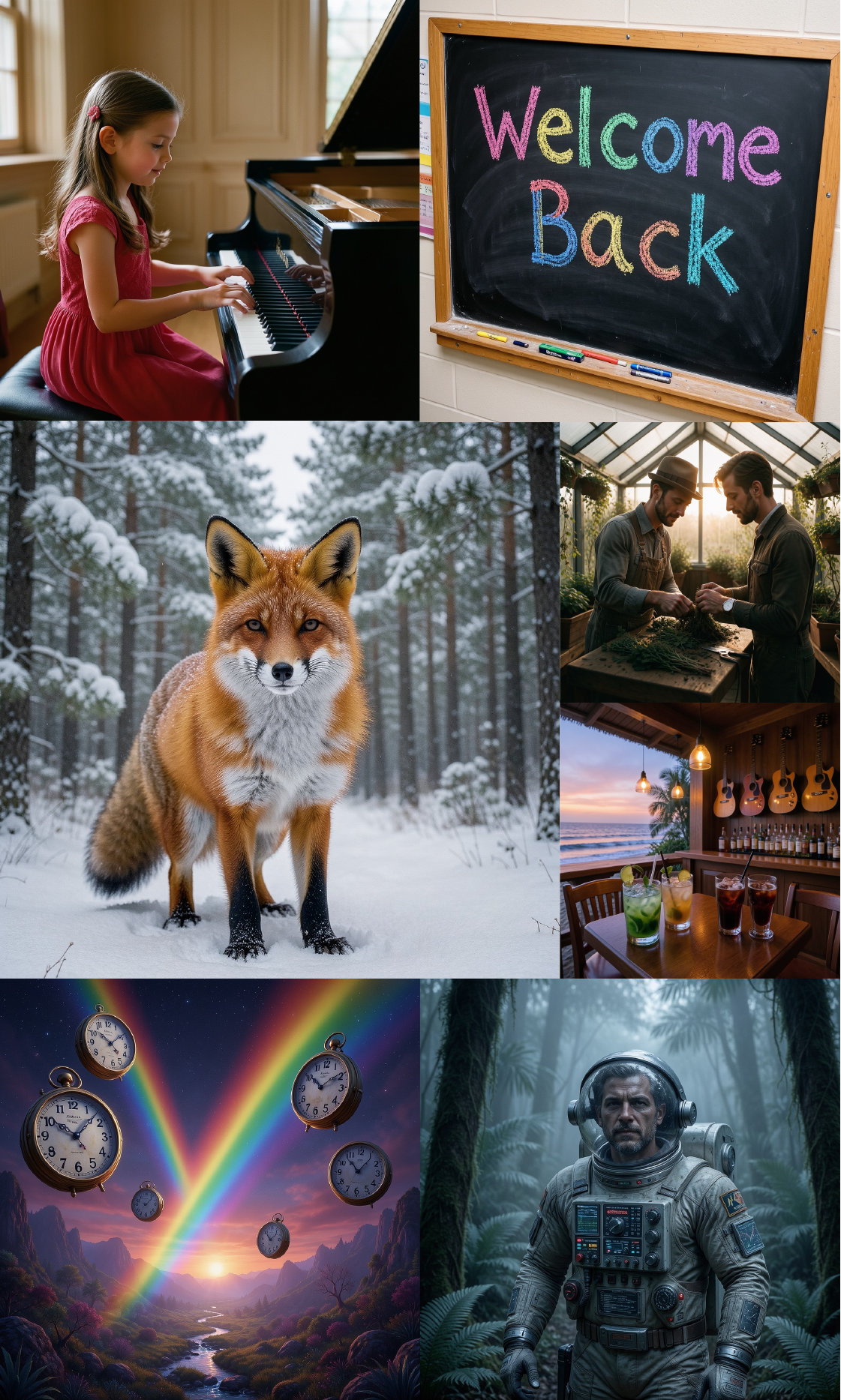}
     \caption{Visual generations produced by our RTDMD method under 4 NFE on FLUX.2 4B~\citep{flux-2} without applying classifier-free guidance (CFG)~\citep{ho2022classifierfreediffusionguidance}.}
    \label{fig:supple_flux}
\end{figure}

\clearpage

\end{document}